\documentclass[a4paper,12pt]{article}

\usepackage[utf8]{inputenc}
\usepackage[unicode=true,hypertexnames=false]{hyperref}
\usepackage{amsmath,amssymb,amsthm,amsfonts,dsfont}
\usepackage{cite}
\usepackage{etoolbox}
\usepackage[table]{xcolor}
\usepackage[linesnumbered,ruled]{algorithm2e}
\usepackage{cite}
\usepackage{graphicx}
\usepackage{url}
\usepackage{tikz}
\usetikzlibrary{arrows,calc,decorations.markings,decorations.pathreplacing,math,arrows.meta}
\usepackage{adjustbox}
\usepackage[capitalise]{cleveref}
\usepackage{xspace}
\usepackage{mathtools}
\usepackage{multirow}

\graphicspath{{../pic/}}

\allowdisplaybreaks

\newcommand{\oea}{$(1 + 1)$~EA\xspace}
\newcommand{\ea}{$(\mu + \lambda)$~EA\xspace}
\newcommand{\commaea}{$(\mu , \lambda)$~EA\xspace}
\newcommand{\ollga}{$(1 + (\lambda , \lambda))$~GA\xspace}
\newcommand{\onemax}{\textsc{OneMax}\xspace}
\newcommand{\xdivk}{\textsc{XdivK}\xspace}
\newcommand{\plateau}{\textsc{Plateau}\xspace}
\newcommand{\jump}{\textsc{Jump}\xspace}
\newcommand{\leadingones}{\textsc{LeadingOnes}\xspace}
\DeclareMathOperator{\mutate}{\textsc{Mutate}}
\DeclareMathOperator{\sign}{sign}

\newcommand{\N}{{\mathbb N}}

\newcommand{\R}{{\mathbb R}}
\newcommand{\reals}{\mathbb{R}}
\newcommand{\norm}[1]{\left\lVert#1\right\rVert}

\newtheorem{theorem}{Theorem}
\newtheorem{lemma}[theorem]{Lemma}
\newtheorem{corollary}[theorem]{Corollary}

\begin{document}

\title{Precise Runtime Analysis for Plateau Functions\footnote{This work is a significantly extended version of the PPSN 2018 paper~\cite{AntipovD18}. It completes the original work by including the mathematical proofs, which were omitted in the conference version for reasons of space, and it extends the conference version by proving the same result for mutation operators with a sub-constant probability to flip exactly one bit, by a tail bound for the runtime, and by a wider selection of applications in Section~\ref{sec:cor}.}}

\author{Denis Antipov\\
ITMO University \\
St. Petersburg \\
Russia \\
and\\
Laboratoire d'Informatique (LIX)\\
CNRS \\
\raisebox{0mm}[0mm][0mm]{\'E}cole Polytechnique \\
Institute Polytechnique de Paris \\
Palaiseau\\
France \\ 
\\
\and
Benjamin Doerr\\
Laboratoire d'Informatique (LIX)\\
CNRS \\
\raisebox{0mm}[0mm][0mm]{\'E}cole Polytechnique \\
Institute Polytechnique de Paris \\
Palaiseau\\
France \\
\\
}

\maketitle

\vspace{1cm}

\begin{abstract}
  To gain a better theoretical understanding of how evolutionary algorithms (EAs) cope with plateaus of constant fitness, we propose the $n$-dimensional $\plateau_k$ function as natural benchmark and analyze how different variants of the $(1 + 1)$~EA optimize it. The  $\plateau_k$ function has a plateau of second-best fitness in a ball of radius $k$ around the optimum. As evolutionary algorithm, we regard the $(1 + 1)$~EA using an arbitrary unbiased mutation operator. Denoting by $\alpha$ the random number of bits flipped in an application of this operator and assuming that $\Pr[\alpha = 1]$ has at least some small sub-constant value, we show the surprising result that for all constant $k \ge 2$, the runtime~$T$ follows a distribution close to the geometric one with success probability equal to the probability to flip between $1$ and $k$ bits divided by the size of the plateau. Consequently, the expected runtime is the inverse of this number, and thus only depends on the probability to flip between $1$ and $k$ bits, but not on other characteristics of the mutation operator.
  Our result also implies that the optimal mutation rate for standard bit mutation here is approximately~$k/(en)$.
  Our main analysis tool is a combined analysis of the Markov chains on the search point space and on the Hamming level space, an approach that promises to be useful also for other plateau problems.
\end{abstract}

\section{Introduction}
\sloppy

This work aims at making progress on several related subjects---we aim at understanding how evolutionary algorithms optimize non-unimodal\footnote{As common in optimization, we reserve the notion \emph{unimodal} for objective functions such that each non-optimal search point has a strictly better neighbor.} fitness functions, what mutation operators to use in such settings, how to analyze the behavior of evolutionary algorithms on large plateaus of constant fitness, and in particular, how to obtain runtime bounds that are precise including the leading constant.

The recent work~\cite{doerr-fast-ga} observed that a large proportion of the theoretical work in the past concentrates on analyzing how evolutionary algorithms optimize unimodal fitness functions and that this can lead to misleading recommendations on how to design evolutionary algorithms. Based on a precise analysis of how the \oea optimizes jump functions, it was observed that the classic recommendation to use standard bit mutation with mutation rate $\frac 1n$ is far from optimal for this function class. For jump size $k$, a speed-up of order $k^{\Theta(k)}$ can be obtained by using a mutation rate of $\frac kn$.

Jump functions are difficult to optimize because the optimum is surrounded by a large set of search points of very low fitness (all search points in Hamming distance $1$ to $k-1$ from the optimum). However, local optima are not the only feature which makes functions difficult to optimize. Another challenge for most evolutionary algorithms are large plateaus of constant fitness. On such plateaus, the evolutionary algorithm learns little from evaluating search points and consequently performs an unguided random walk. To understand this phenomenon in more detail, we propose a class of fitness functions very similar to jump functions. A \emph{plateau function} with plateau parameter $k$ is identical to a jump function with jump size $k$ except that the $k-1$ Hamming levels around the optimum do not have a small fitness, but have the same second-best fitness as the $k$-th Hamming level. Consequently, these functions do not have true local optima (in which an evolutionary algorithm could get stuck for longer time), but only a plateau of constant fitness. Our hope is that this generic fitness function with a plateau of scalable size may aid the understanding of plateaus in evolutionary computation in a similar manner as the jump functions have led to many useful results about the optimization of functions with true local optima, e.g.,~\cite{droste-ea,Jansen2002,doerr-jump,unrestricted-jump-evco,dang-crossover,FriedrichKKNNS16,CorusOY17,CorusOY18,doerr-fast-ga,DangFKKLOSS18,WhitleyVHM18,HasenohrlS18,Doerr19a,Doerr19b}.

When trying to analyze how evolutionary algorithms optimize plateau functions, we observe that the active area of theoretical analyses of evolutionary algorithms has produced many strong tools suitable to analyze how evolutionary algorithms make true progress (e.g., various form of the fitness level method~\cite{wegener-theory,sudholt-levels,lehre-popul-journal,CorusDEL18} or drift analysis~\cite{he-yao-drift-intro,multiplicative-drift-theorem,lehre-var-drift,DoerrK19}), but much less is known on how to analyze plateaus.
This is not to mean that plateaus have not been analyzed previously, see, e.g.,~\cite{garnier-binary,JansenW01,doerr-assymetric-mut-op,brockhoff,FriedrichHN09,NeumannSW09,FriedrichHN10},
but these results appear to be more ad hoc and less suitable to derive generic methods for the analysis of plateaus. In particular, with the exception of~\cite{garnier-binary}, we are not aware of any results that determine the runtime of an evolutionary algorithm on a fitness function with non-trivial plateaus precisely including the leading constant (whereas a decent number of very precise results have recently appeared for unimodal fitness functions, e.g.,~\cite{bottcher-leading-ones,doerr-sharp-bounds,witt-linear-functions,LissovoiOW17,HwangPRTC18,DoerrDL19,HwangW19}).

Such precise results are necessary for our further goal of understanding the influence of the mutation operator on the efficiency of the optimization process. Mutation is one of the most basic building blocks in evolutionary computation and has, consequently, received significant attention also in the runtime analysis literature. We refer to the discussion in~\cite{doerr-fast-ga} for a more extensive treatment of this topic and only note here that even small changes of the mutation operator or its parameters can lead to a drastic change of the efficiency of the algorithm~\cite{doerr-global-local,doerr-mutation-rate-matters}.

\textbf{Our results:} Our main result is a very general analysis of how the simplest mutation-based evolutionary algorithm, the \oea, optimizes the $n$-dimensional plateau function with plateau parameter $k \in \N$, which is considered as a constant and does not depend on $n$ when $n$ tends to the positive infinity. We allow the algorithm to use any unbiased mutation operator (including, e.g., one-bit flips, standard bit mutation with an arbitrary mutation rate, or the fast mutation operator of~\cite{doerr-fast-ga}) as long as the operator flips exactly one bit with probability $\omega(n^{-\frac{1}{2k - 2}})$. This assumption is natural, but also ensures that the algorithm can reach all points on the plateau. Denoting the number of bits flipped in an application of this operator by the random variable~$\alpha$, we prove that the expected optimization time (number of fitness evaluations until the optimum is visited) is
\[\frac{n^k}{k!\Pr[1 \le \alpha \le k]}(1 + o(1)).\]
This result, tight apart from lower order terms only, is remarkable in several respects. It shows that the performance depends very little on the particular mutation operator, only the probability to flip between $1$ and $k$ bits has an influence. The absolute runtime is also surprising --- it is the size of the plateau times the waiting time until we flip between $1$ and $k$ bits.

A similar-looking result was obtained in~\cite{garnier-binary}, namely that the expected runtime of the \oea with 1-bit mutation and with standard bit mutation with rate $\frac 1n$ on the needle function is (apart from lower order terms) the size of the plateau times the probability to flip a positive number of bits (which is $1$ for 1-bit mutation and $(1-o(1))(1-\frac 1e)$ for standard bit mutation with rate $\frac{1}{n}$).
Our result is different from that one in that we consider constrained plateaus of arbitrary (constant) radius $k \ge 2$, and more general in that we consider a wide class of unbiased mutation operators. Despite the difference in the plateaus, the expected runtime is surprisingly similar, which is the size of the plateau times the expected number of iterations until we flip between $1$ and $k$ bits (where for the needle function we can take $k = n$).

We note that there is a substantial difference between the case $k=n$ and $k$ constant. Since the needle function consists of a plateau containing the whole search space apart from the optimum, the optimization time in this case is just the hitting time of a particular search point when doing an undirected random walk (via repeated mutation) on the hypercube $\{0,1\}^n$. For $\plateau_k$ with constant~$k$, the plateau has a large boundary. More precisely, almost all\footnote{in the usual asymptotic sense, that is, meaning all but a lower order fraction} search points of the plateau lie on its outer boundary and furthermore, all these search points have almost all their neighbors outside the plateau. Hence a large number of iterations (namely almost all) are lost in the sense that the mutation operator generates a search point outside the plateau (and different from the optimum), which is not accepted. Interestingly, as our result shows, the optimization of such restricted plateaus is not necessarily significantly more difficult (relative to the plateau size) than the optimization of the unrestricted needle plateau.

Our precise runtime analysis allows to deduce a number of particular results. For example, when using standard bit mutation, the optimal\footnote{We call a mutation rate optimal when it delivers an expected runtime that differs from the truly optimal one 
at most by lower order terms, that is, e.g. a factor of $(1 \pm o(1))$. This suggests that there might be a range of optimal rates, however without proof we note that changing the mentioned optimal mutation rate by a factor of $(1 \pm \Omega(1))$ would also increase the runtime by a $(1 + \Omega(1))$ factor.}
mutation rate is $\frac{\sqrt[k]{k!}}n$, that is, approximately $\frac k{en}$. This is by a constant factor less than the optimal rate of $\frac kn$ for the jump function with jump size $k$, but again a factor of $\Theta(k)$ larger than the classic recommendation of~$\frac 1n$, which is optimal for many unimodal fitness functions. Hence our result confirms that optimal mutation rates can be significantly higher for non-unimodal fitness functions. While the optimal mutation rates for jump and plateau functions are similar, the effect of using the optimal rate is very different. For jump functions, an $k^{\Theta(k)}$ factor speed-up (compared to the standard recommendation of $\frac 1n$) was observed, here the influence of the mutation operator is much smaller, namely the factor $\Pr[1 \le \alpha \le k]$,
which is trivially at most $1$, but which was assumed to be at least some positive constant.
Interestingly, our results imply that the fast mutation operator described in~\cite{doerr-fast-ga} is not more effective than other unbiased mutation operators, even though it was proven to be significantly more effective for jump functions~\cite{doerr-fast-ga} and it has shown good results in some practical problems~\cite{MironovichB17}.

So one structural finding, which we believe to be true for larger classes of problems and which fits to the result~\cite{garnier-binary} for needle functions, is that the mutation rate, and more generally, the particular mutation operator which is used, is less important while the evolutionary algorithm is traversing a plateau of constant fitness.

The main technical novelty in this work is that we model the optimization process via two different Markov chains describing the random walk on the plateau, namely the chain defined on the $\Theta(n^k)$ elements of the plateau (plus the optimum) and the chain obtained from aggregating these into the total mass on the Hamming levels. Due to the symmetry of the process, one could believe that it suffices to regard only the level chain. The chain defined on the elements, however, has some nice features which the level chain is missing, among others, a symmetric transition matrix (because for any two search points $x$ and $y$ on the plateau, the probability of going from $x$ to $y$ is the same as the probability of going from $y$ to $x$). This symmetry allows us to analyse the speed of convergence to some distribution over the points of the plateau by using some ideas similar to the ones used in~\cite{Vitanyi00} for the analysis of the rapidly mixing Markov chains.
For this reason, we find it fruitful to switch between the two chains. Exploiting the interplay between the two chains and using classic methods from linear algebra, we find the exact expression for the expected runtime.

The most valuable insight given by this approach is that the mixing of the probability mass over the plateau is very fast. More precisely, we show that independently of the first position on the plateau, in slightly more than $\Theta(\sqrt{n}\log(n))$ iterations we are almost equally likely to be at any point of the plateau. A similar mixing argument was used to prove the upper bound on the runtime of the \oea on the \leadingones with strong prior noise in~\cite{Sudholt20}. There, however, only an exponential mixing time was shown, although the author conjectures that it should be polynomial. Our analysis based on the interplay of two Markov chains is problem-specific (e.g., we base our arguments on the symmetry of the plateau), but we are optimistic that the observed behavior of a small mixing time can be also seen on other plateaus which are not too easy to leave.

The rest of the paper has the following structure. In Section~\ref{statement} we describe the \oea, the operators it uses and the problem on which we analyse the algorithm. In Section~\ref{sec:notation} we list the mathematical means that are used in our analysis. We also introduce the central tool of our analysis --- the two Markov chains, show their properties and the connection between the two chains. In Section~\ref{runtime} we prove the main result of this work, which is, the precise runtime of the \oea on the $\plateau_k$ function for constant $k$. The corollaries from the main result, which are, the precise runtime of different variants of the \oea, are shown in Section~\ref{sec:cor}. Finally, we summarize the results in Section~\ref{conclusion}.   

\section{Problem Statement}\label{statement}

We consider the maximization of a function defined on the space of bit-strings of length $n$ which resembles the \onemax function, but has a plateau of second-highest fitness of radius $k$ around the optimum. We call this function $\plateau_k$ and define it as follows.

\begin{align*}
  \plateau_k(x) :=
  \begin{cases}
    n - k, &\text{ if } n - k < \onemax(x) < n, \\
    \onemax(x), &\text{ otherwise,}
  \end{cases}
\end{align*}
where $\onemax(x) \coloneqq \norm{x}_1$ is
the number of one-bits in $x \in \{0, 1\}^n$.

Notice that the plateau of the function $\plateau_k(x)$ consists of all bit-strings that have at least $n - k$ one-bits, except the optimal bit-string $x^* = (1,\dots,1)$. See Fig.~\ref{fig:plateau} for an illustration of $\plateau_k$.

\begin{figure}
  \begin{center}
   \begin{tikzpicture}

    \draw [dashed] (0, 3.5) -- (5, 3.5) -- (5, 0);
    \draw [dashed] (0, 2.8) -- (4, 2.8) -- (4, 0);

    \draw [{<[scale=1.5]}-{>[scale=1.5]}] (0, 4) -- (0, 0) -- (6, 0);
    \draw [very thick] (0, 0) -- (4, 2.8) -- (5, 2.8);

    \draw [fill=white] (5, 2.8) circle (0.7mm);
    \draw [fill=black] (5, 3.5) circle (0.7mm);

    \node [above] at (6.7, -0.7) {$\onemax(x)$};
    \node [above] at (5, -0.5) {$n$};
    \node [above] at (4, -0.535) {$n - k$};
    \node [left] at (0, 2.8) {$n - k$};
    \node [left] at (0, 3.5) {$n$};
    \node [left] at (0, 4) {$\plateau_k(x)$};

   \end{tikzpicture}
  \end{center}
  \caption{Plot of the \plateau function. As a function of unitation, the function value of a search point $x$ depends only on the number of one-bits in~$x$.}
  \label{fig:plateau}
\end{figure}

To compare the results of our analysis to the best runtime which could be obtained by an algorithm using only unbiased operators, we note that the \emph{unary unbiased black-box complexity} (see~\cite{LehreW12} for the definition) of $\plateau_k$ is $\Theta(n \log n)$ for all constants $k$. While this implies that there is a unary unbiased black-box algorithm finding the optimum of $\plateau_k$ in $O(n \log n)$ time, such results generally do not indicate that a problem is easy for reasonable evolutionary algorithms. For example, in~\cite{DoerrDK14artint} it was shown that the  NP-complete partition problem also has a unary unbiased black-box complexity of $O(n \log n)$.

\begin{lemma}
  For all constants $k$, the unary unbiased black-box complexity of the $\plateau_k$ function is $\Theta(n \log n)$.
\end{lemma}

\begin{proof}
  The lower bound follows from the $\Omega(n \log n)$ lower bound for the unary unbiased black-box complexity of $\onemax$ shown in~\cite{LehreW12}. Since we can write $\plateau_k = f \circ \onemax$ for a suitable function $f$ (such that $f(x) = x$, if $x \notin [n-k..n]$ and $f(x) = n - k$ otherwise), any algorithm solving $\plateau_k$ can be transferred into an algorithm which treats all points with fitness in $[n - k..n - 1]$ as points with fitness $(n - k)$ and therefore solving $\onemax$ in the same time.

  The upper bound follows along the same lines as the $O(n \log n)$ upper bound for the unary unbiased black-box complexity of $\jump_k$, see~\cite{doerr-jump} and note that the algorithm given there contains a sub-routine which, in expected constant time, for a given constant radius $r$ determines the Hamming distance $H(x,x^*)$ of a point $x$ from the optimum $x^*$ without evaluating search points $y$ with $H(y,x^*) \le r$. Note that the Hamming distance from the optimum determines the $\onemax$ value of $x$. Hence with this routine one can optimize both jump and plateau functions by simulating an $O(n \log n)$ black-box algorithm for $\onemax$.
\end{proof}

To understand how evolutionary algorithms optimize plateau functions, we consider the most simple evolutionary algorithm, the \oea shown in Algorithm~\ref{alg:oea}. However, we allow the use of an arbitrary unbiased mutation operator. A mutation operator $\mutate$ for bit-string representations is called \emph{unbiased} if it is symmetric in the bit-positions $[1..n]$ and in the bit-values $0$ and $1$. This is equivalent to saying that for all $x \in \{0,1\}^n$ and all automorphisms $\sigma$ of the hypercube $\{0,1\}^n$ (respecting Hamming neighbors) we have $\sigma^{-1}(\mutate(\sigma(x)) = \mutate(x)$, which is an equality of distributions. The notation of unbiasedness was introduced (also for higher-arity operators) in the seminal paper~\cite{LehreW12}.

\begin{algorithm}
	$x \gets $ random bit string of length $n$\; 
    \While{not terminated}
        {
        $y \gets \mutate(x)$\;
        \If{$f(y) \ge f(x)$}
            {
             $x \gets y$\;   
            }
        }
		\caption{The fast \oea with a generic mutation operator maximizing $f:\{0,1\}^n\to\R$}
		\label{alg:oea}
\end{algorithm}

For our purposes, it suffices to know
that the set of unbiased mutation operators consists of all operators which can be described as follows. First, we choose a number $\alpha \in [0..n]$ according to some probability distribution and then we flip exactly $\alpha$ bits chosen uniformly at random. Examples for unbiased operators are the operator of Random Local Search, which flips a single random bit, or standard bit mutation, which flips each bit independently with probability $\frac{1}{n}$. Note that in the first case $\alpha$ is always equal to one, whereas in the latter $\alpha$ follows a binomial distribution with parameters $n$ and~$\frac 1n$. This characterization can be derived from~\cite[Proposition~19]{DoerrKLW13tcs}. It was explicitly stated in~\cite{DoerrDY20}.

\paragraph{Additional assumptions:} The class of unbiased mutation operators contains a few operators which are unable to solve even very simple problems. For example, operators that always flips exactly two bits never finds the optimum of any function with unique optimum if the initial individual has an odd Hamming distance from the optimum. To avoid such artificial difficulties, we only consider unbiased operators that have at least $\omega(n^{-\frac{1}{2k-2}})$ probability to flip exactly one bit.

As usual in runtime analysis, we are interested in the optimization behavior for large problem size $n$. Formally, this means that for each fixed $k$ we view the runtime $T_k = T_k(n)$ as a function of $n$ and aim at understanding its asymptotic behavior for $n$ tending to infinity. We aim at sharp results (including finding the leading constant), that is, we try to find a simple function $\tau_k: \N \to \R$ such that $T_k(n) = (1+o(1)) \tau_k(n)$, which is equivalent to saying that $\lim_{n \to \infty} T_k(n)/\tau_k(n) = 1$. In this limit sense, however, we treat $k$ as a constant, that is, $k$ is a given positive integer and not also a function of $n$.

Finally, since the case $k=1$ is well-understood ($\plateau_1$ is the well-known $\onemax$ function), we always assume $k \ge 2$.

\section{Preliminaries and Notation}\label{sec:notation}

\subsection{Tools from Linear Algebra}\label{sec:linear_algebra}

In this section we brief{}ly review the terms, tools and facts from the linear algebra that we use in this work.

We use $\N$ to denote the set of all positive integer numbers and we use $\N_0$ to denote $\N \cup \{0\}$. We denote the vector of length $n$ that consists only of ones by $1^n$ and the vector of length $n$ that consists only of zeros by $0^n$.

Given the square matrix $A$, the vector $x$ is called the \emph{left eigenvector} of the matrix $A$ if $xA = \lambda x$ for some $\lambda \in \mathbb{C}$. In this situation, $\lambda$ is called \emph{eigenvalue} of the matrix $A$. The vector $x$ is called \emph{right eigenvector} if $Ax = \lambda x$ for some $\lambda \in \mathbb{C}$. Since in this work we regard only left eigenvectors, we call them just \emph{eigenvectors}.

The \emph{spectrum} of a matrix is the set of all its eigenvalues. If a matrix has size $n \times n$, then the number of its eigenvalues is not greater than $n$. For each eigenvalue there exists a corresponding \emph{eigenspace}, that is, the linear span of all the eigenvectors that correspond to the eigenvalue.

The only point shared by any two eigenspaces that correspond to two different eigenvalues is $0^n$.

The \emph{characteristic polynomial} $\chi(\lambda)$ of matrix $A$ is the function of $\lambda$ that is defined as the determinant of the matrix $A-\lambda I$, where $I$ is the identity matrix. The set of roots of the characteristic polynomial equals the spectrum of the matrix $A$.

The inner product of the vectors $x = (x_0, \dots, x_{n - 1})$ and $y = (y_0, \dots, y_{n - 1})$ is a scalar value defined by $\langle x, y \rangle = \sum_{i = 0}^{n - 1}x_i y_i.$
The two vectors are \emph{orthogonal} if their inner product is zero.

For every \emph{diagonalizable} matrix $A$ of size $n \times n$ there exists a set $\{e^i\}_{i = 0}^{n - 1}$ of eigenvectors that form a \emph{basis} of $\reals^n$.
A basis is called \emph{orthogonal} when all pairs of the basis vectors are orthogonal.
A matrix $A = (a_i^j)$ is \emph{symmetric} if for every $i$ and $j$ we have $a_i^j = a_j^i$.

We use the following two properties of symmetric matrices.

\begin{lemma}\label{lm:sym-eigenvalues}
All eigenvalues of a symmetric matrix are real.
\end{lemma}
\begin{lemma}\label{lm:sym-eigenvectors}
  Two eigenvectors of a symmetric matrix that correspond to different eigenvalues are orthogonal. Also every symmetric matrix of size $n \times n$ is diagonalizable, which means that there exists an orthogonal basis of $\reals^n$ which consist of eigenvectors of this matrix.
\end{lemma}

In this work we also encounter \emph{irreducible} matrices. Among the several definitions, the following is the easiest to check for the non-negative matrices considered in this work. For each non-negative matrix $A$ of size $n \times n$ we can build a directed graph $G_A$ by taking an empty graph on $n$ vertices and adding for each non-negative component $a_i^j$ of $A$ an edge from vertex $i$ to vertex $j$. 
Then a matrix $A$ is \emph{irreducible} if and only if graph $G_A$ is strongly connected.

For example, the transition matrix of an irreducible Markov chain (a chain such that each state is reachable from each other state) is irreducible.

A crucial role in this work is played by the Perron-Frobenius theorem~\cite{perron-frobenius}. This theorem gives a series of properties of the irreducible matrices, among them we use the following four.
\begin{theorem}[Perron-Frobenius]\label{lm:p-f}
  Any irreducible non-negative matrix $A$ has the following properties.
\begin{itemize}
  \item The largest eigenvalue $\lambda_0$ of $A$ lies between the minimal and the maximal row sum of $A$.
  \item For every eigenvalue $\lambda$ of $A$ different from the largest eigenvalue~$\lambda_0$ we have $|\lambda| < \lambda_0$.
  \item The largest eigenvalue of $A$ has a one-dimensional eigenspace.
  \item There exists an eigenvector which corresponds to the largest eigenvalue $\lambda_0$ all components of which are strictly positive.
\end{itemize}
\end{theorem}

When talking about vector norms, we use the following notation. For any $p \in (0, +\infty)$ and any vector $x \in \reals^n$, we let
\begin{align*}
  \norm{x}_p = \left(\sum_{j = 0}^{n - 1} |x_j|^p\right)^{1/p}.
\end{align*}
In this work we use only the Manhattan norm ($p = 1$) and the Euclidean norm ($p = 2$). We use the following properties of these norms.

\begin{lemma}\label{lm:norm-manhattan-euclidean}
For all $x \in \reals^n$ we have
\begin{align*}
    \norm{x}_2 \le \norm{x}_1 \le \sqrt{n} \norm{x}_2.
\end{align*}
\end{lemma}

The following lemma is often called \emph{triangle inequality}
\begin{lemma}\label{lm:triangle}
  For any norm $\norm{\cdot}$ and for every $x$, $y$ and $z = x + y$ we have
  \begin{align*}
    \norm{x} - \norm{y} \le \norm{z} \le \norm{x} + \norm{y}
  \end{align*}
\end{lemma}

We use the following properties of the Euclidean norm.
\begin{lemma}\label{lm:euclidean-norm-1}
  If vectors $x^1, \dots, x^n$ are orthogonal, then for any values $a_1, \dots, a_n \in \reals$ we have
  \begin{align*}
    \norm{\sum_{i = 1}^n a_i x^i}_2 \le \max_{i \in [1..n]}|a_i| \norm{\sum_{i = 1}^n x^i}_2
  \end{align*}
\end{lemma}
\begin{lemma}\label{lm:euclidean-norm-2}
  If vectors $x^1, \dots, x^n$ are orthogonal, then for any subset $S \subset [1..n]$ we have
  \begin{align*}
    \norm{\sum_{i \in S} x^i}_2 \le \norm{\sum_{i = 1}^n x^i}_2
  \end{align*}
\end{lemma}

We also make a use of the orthogonal projection of vectors, which is defined as follows.
Suppose we have vector $x \in \reals^n$ and it is decomposed into the sum of $m$ orthogonal vectors $\{x^i\}_{i = 0}^{m - 1}$ where $m \le n$. Then $x^i$ is the \emph{orthogonal projection} of $x$ to the linear span of $x^i$.
To calculate precisely the norm of the projection, we use the following lemma.
\begin{lemma}\label{lm:projection}
If $x^i$ is the orthogonal projection of $x$, then for any norm $\norm{\cdot}$ we have
\begin{align*}
  \norm{x^i} = \frac{\langle x, \frac{x^i}{\norm{x^i}}\rangle}{\langle \frac{x^i}{\norm{x^i}}, \frac{x^i}{\norm{x^i}}\rangle}.
\end{align*}
\end{lemma}

We also encounter the \emph{self-adjoint operators}.
An operator $A: \reals^n \to \reals^n$ is called \emph{self-adjoint} if for all $x \in \reals^n$ and $y \in \reals^n$ we have $\langle Ax, y \rangle = \langle x, Ay \rangle$, where $\langle \cdot, \cdot \rangle$ stands for the standard inner product.
The operator in this space is self-adjoint if and only if its matrix is symmetric. The most important properties of self-adjoint operators are stated in the Hilbert-Schmidt theorem~\cite{hylbert-schmidt}. We use only one of them.

\begin{lemma}\label{lm:h-sch}
For any self-adjoint operator $A: \reals^n \to \reals^n$ there exists an orthonormal basis of $\reals^n$ that consists of the eigenvectors of $A$.
\end{lemma}

\subsection{Absorbing Markov Chains\protect\footnote{In this subsection we use a standard notation for the absorbing Markov chains such as $N$ for the fundamental matrix, $P$ for the transition matrix and $Q$ for the transient-to-transient transition matrix. In the rest of the paper for the reader's convenience we redefine these common and easy-to-remember symbols to denote the objects we work with most frequently.}}\label{sec:leaky_chains}

Markov chains are a widely used tool for the runtime analysis of evolutionary algorithms (see, e.g.,~\cite{Muhlenbein93,Suzuki95,Rudolph96}). In this work we only regard \emph{absorbing} Markov chains. A Markov chain is called absorbing if there is a subset $S'$ of the set of its states $S$ such that 
\begin{enumerate}
  \item [(1)] for every state $s_1 \in S$ there exists a state $s_2 \in S'$ such that there exists a path of transitions with positive probabilities from $s_1$ to $s_2$ (we call $s_2$ an \emph{absorbing} state) and
  \item [(2)] for every absorbing state $s \in S'$ the probability to leave this state is zero.
\end{enumerate}
The non-absorbing states (the states in $S \setminus S'$) are called the \emph{transient states}.

Absorbing chains appear naturally in runtime analysis. When taking as states of the Markov chain the possible states of the algorithm, we can assume the optima to be absorbing.
The runtime of the algorithm is the number of transitions in the chain until it reaches an absorbing state.
We only regard absorbing Markov chains with exactly one absorbing state.

The standard way to compute the expected number of steps until an absorbing state is reached uses the \emph{fundamental matrix}, which is built as follows. Let $P$ be the transition matrix of an absorbing Markov chain, that is, the matrix where each element $p_i^j$ is equal to the transition probability from state $i$ to state $j$. Let $Q$ be the square submatrix of $P$ consisting only of the rows and columns which correspond to transient states of the chain. We call $Q$ the \emph{transient matrix} for brevity. Then the fundamental matrix $N$ of this chain is defined as
\begin{align*}
  N = \sum_{t = 0}^{+\infty}Q^t = (I - Q)^{-1},
\end{align*}
where $I$ is the identity matrix of the same order as $Q$. Let $\pi$ be a stochastic vector which represents the initial distribution over the transient states. Then the expected time until we reach an absorbing state is 
\begin{align*}
  E[T] = \pi N \mathds{1} = \norm{\pi N}_1,
\end{align*}
where $\mathds{1}$ is a column vector of all ones.

However, working with the fundamental matrix is not convenient, since it might be hard to compute its elements precisely. Instead, in this paper we study the properties of the transient matrix $Q$ and compute the expected time until the absorption as
\begin{align}\label{eq:expectation-distribution}
  E[T] = \norm{\sum_{t = 0}^{+\infty} \pi Q^t}_1 = \sum_{t = 0}^{+\infty} \norm{\pi Q^t}_1,
\end{align}
where the last equation is satisfied since all components of vectors $\pi Q^t$ are non-negative. Another way to derive this equation for the expected runtime is to use the formula for the expectation of a non-negative integer-valued random variable, which is,
\begin{align*}
  E[T] = \sum_{t = 0}^{+\infty} \Pr[T \ge t].
\end{align*} 
Note that $\Pr[T \ge t]$ is the probability that we are in a transient state in the start of iteration $t$, which is $\norm{\pi Q^t}_1$.
We show this approach to be much more fruitful, since after finding some properties of the spectrum of $Q$ it allows us to use the decomposition of $\pi$ into the sum of eigenvectors of $Q$ to obtain precise estimates on the runtime.
  
\subsection{Two Markov Chains}\label{sec:model}
For the optimization process of our \oea we first observe that, since the unbiased operator with constant probability flips exactly one bit, the expected time to reach the plateau is $O(n\log n)$. Since the time for leaving the plateau (as shown in this paper) is $\Omega(n^k)$, we only consider the runtime of the algorithm after it has reached the plateau.

For this runtime analysis on the plateau we consider the plateau in two different ways. The first way is to regard a Markov chain that contains $N + 1$ states, where $N = \sum_{i = 0}^{k - 1} \binom{n}{k - i}$. Each state represents one element of the plateau plus there is one absorbing state for the optimum. Note that $N = \frac{n^k}{k!} + o(n^k)$, since $\binom{n}{j} = \frac{n^j}{j!}(1 + o(1))$ for all $j \in [1..k]$.
The transition probability from transient state $x$ to any state $y$ is $q_x^y = \Pr[\alpha = d]\binom{n}{d}^{-1}$, where $d$ is the Hamming distance between $x$ and $y$.
This implies that the transition probability from $x$ to $y$ is equal to the transition probability from $y$ to $x$ for any pair of the transient states. Therefore, the transient matrix is symmetric, which gives us the opportunity to use Lemma~\ref{lm:sym-eigenvalues} and Lemma~\ref{lm:sym-eigenvectors}.
We call this Markov chain the \emph{individual chain}, denote its transient matrix by $Q$ and call the space of real vectors of dimension $N$ the \emph{individual space}\footnote{Note that the dimension of the individual space is equal to the number of transient states of the individual chain, not to the total number of states. Hence, matrix $Q$ defines a linear operator on the individual space.}, since the current state of the chain defines the current individual of the algorithm.

To define the second Markov chain, we first define the \emph{$i$-th level} as the set of all search points that have exactly $n - k + i$ one-bits. Then the plateau is the union of levels $0$ to $k - 1$ and the optimum is the only element of level~$k$. 
Notice that the $i$-th level contains exactly $\binom{n}{k - i}$ elements (search points). For every $i, j \in [0..k]$ we have that for any element of the $i$-th level the probability to mutate to the $j$-th level is the same due to the unbiasedness of the operator. 
Therefore we can regard a Markov chain of $k + 1$ states, where the $i$-th state ($i \in [0..k]$) represents the elements of the $i$-th level. State $k$ is an absorbing state.
The transition probability from level $i$ to level $j$ is
\begin{align}\label{eq:trans_prob}
  p_i^j = \begin{cases}
    0, \text{ if } i = k, j \ne k, \\
    \sum\limits_{m = 0}^{k - j} \binom{k - i}{j - i + m} \binom{n - k + i}{m} \binom{n}{j - i + 2m}^{-1}\Pr[\alpha = j - i + 2m] , \text{ if } j > i, \\
      \sum\limits_{m = 0}^{k - i} \binom{k - i}{m} \binom{n - k + i}{i - j + m} \binom{n}{i - j + 2m}^{-1} \Pr[\alpha = i - j + 2m] ,    \text{ if } j < i \text{ and } i \ne k, \\
      1 - \sum\limits_{m = 0, m \ne i}^k p_i^m, \text{ if } j = i , \\
  \end{cases}
\end{align}

where we assume that $n > 2k$ not to complicate the upper limit of sums. This assumption is justified by that we only consider constant $k$ and we estimate the runtime with $n$ tending to infinity. We notice the following useful property of these probabilities.
\begin{lemma}\label{lem:backward_prob}
For all $i, j \in [0..k - 1]$ we have
\[
\binom{n}{k - i}p_i^j = \binom{n}{k - j}p_j^i.
\]
\end{lemma}
\begin{proof}
  Let $L_s$ denote level $s$ for all $s \in [0..k-1]$. Let also $p_{x \to L_s}$ denote the probability to get from individual $x$ to any individual in level $s$. Since for all individuals $x$ in level $i$ the probability $p_{x \to L_j}$ is the same and equal to $p_i^j$ and since there are $\binom{n}{k - i}$ individuals in level $i$, we have
  \begin{align*}
    \binom{n}{k - i}p_i^j = \sum_{x \in L_i} p_{x \to L_j} = \sum_{x \in L_i} \sum_{y \in L_j} q_x^y = \sum_{x \in L_i} \sum_{y \in L_j} q_y^x = \sum_{y \in L_j} p_{y \to L_i} = \binom{n}{k - j}p_j^i.
  \end{align*}
\end{proof}

We observe that the probability to gain $\ell$ levels is $O(n^{-\ell})$.
\begin{lemma}\label{lm:probs}
For all $i \in [0..k-1]$ and $j \in [i + 1..k]$, we have $p_i^j = O(n^{-(j - i)})$.
\end{lemma}
\begin{proof}
  By Lemma~\ref{lem:backward_prob} and since $p_j^i \le 1$ we have
  \begin{align*}
    p_i^j = \frac{\binom{n}{k - j}p_j^i}{\binom{n}{k - i}} \le \frac{(n - k + i)! (k - i)!}{(n - k + j)!(k - j)!} = O(n^{-(j - i)}).
  \end{align*}
\end{proof}

We call this Markov chain the \emph{level chain} and we call the space of real vectors of length $k$ the \emph{level space}\footnote{As well as for the individual space, the dimension of the level space is equal to the number of the transient states of the level chain and matrix $P$ defines a linear operator on this space.}. The level chain is illustrated in Fig.~\ref{fig:level}.
The transient matrix $P$ of the level chain has a size of $k \times k$.
The matrix $P$ (unlike $Q$) is not symmetric. In our analysis we use the following property of the matrix $P$.
 \begin{lemma}\label{lm:rowsum}
The sum of each row of $P$ is $1 - O(\frac 1n)$.
 \end{lemma}
\begin{proof}
The sum of the $i$-th row of $P$ is
\begin{align}\label{eq:rowsum}
\sum_{j = 0}^{k - 1} p_i^j = 1 - p_i^k,
\end{align}
since the sum of all the outgoing probabilities for each state in the original Markov chain is one.
By Lemma~\ref{lm:probs} we have $p_i^k = O(n^{-(k - i)}) = O(\frac 1n)$.
\end{proof}

\begin{figure}
  \begin{center}
   \includegraphics[width=0.7\textwidth]{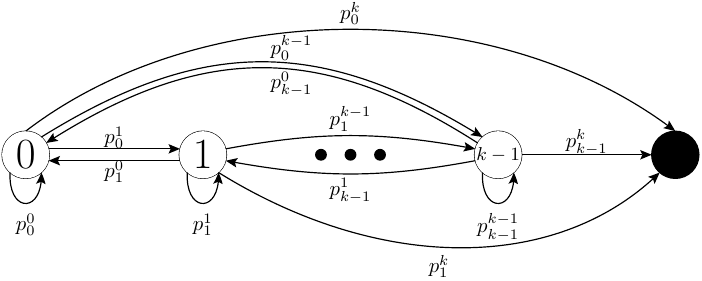}
  \end{center}
  \caption{Illustration of the level chain. The black circle represents the optimum that is an absorbing state. The states $[0..k-1]$ represent the levels of the plateau surrounding the optimum.}
  \label{fig:level}
\end{figure}

There is a natural mapping from the level space to the individual space. Every vector $x = (x_0, \dots, x_{k - 1})$ can be mapped to the vector $\phi(x) = (y_0, \dots, y_{N - 1})$, where $y_i = x_j / \binom{n}{k - j}, $ if the $i$-th element belongs to the $j$-th level.
If $x$ is a distribution over the levels, that is, $x \in [0,1]^k$ and $\|x\|_1 = 1$, then $\phi(x)$ is the distribution over the elements of the plateau which is uniform on the levels and which has the same total mass on each level as~$x$.
This mapping has several useful properties.

\begin{lemma}\label{lm:phi-linear}
$\phi$ is linear, that is, we have $\phi(\alpha x + \beta y) = \alpha \phi(x) + \beta \phi(y)$ for all $x, y \in \reals^k$ and all $\alpha, \beta \in \reals$.
\end{lemma}
This property follows directly from the definition of $\phi$.

\begin{lemma}\label{lm:phi-p}
  For all $x \in \reals^k$ we have $\phi(x P) = \phi(x) Q$.
\end{lemma}
\begin{proof}
  In informal words, this property holds because both matrices $P$ and $Q$ represent the same operator, but in different spaces. Thus, the result of applying this operator to some vector and then switching the space is the same as performing these two actions in a reversed order.

  For the formal proof, recall that \emph{level $i$} is the set of all individuals in distance $(k - i)$ from the optimum. We use the fact that for any individual $m$ in level $j$ we have
  \[
  p_j^i = \sum_{\ell \in \text{level }i} q_m^\ell,
  \]
  where $q_m^\ell$ is the element of matrix $Q$, that is, the probability to obtain individual $\ell$ from individual $m$. From this and from the definition of $\phi$ we calculate the $m$-th element of $\phi(xP)$, assuming that individual $m$ belongs to level $j$.
  \begin{align*}
    (\phi(xP))_m = \frac{(xP)_j}{\binom{n}{k - j}} = \frac{\sum_{i = 0}^{k - 1} x_i p_i^j}{\binom{n}{k - j}}.
  \end{align*}
  By Lemma~\ref{lem:backward_prob} we have $\binom{n}{k - i}p_i^j = \binom{n}{k - j}p_j^i$. Therefore,
  \begin{align*}
    (\phi(xP))_m = \sum_{i = 0}^{k - 1} \frac{x_i}{\binom{n}{k - i}} p_j^i = \sum_{i = 0}^{k - 1} \frac{x_i}{\binom{n}{k - i}} \sum_{\ell \in \text{level }i} q_m^\ell.
  \end{align*}
  Recall that $q_m^\ell = q_\ell^m$ for all $\ell, m \in [0..N - 1]$. Hence, we have
  \begin{align*}
    (\phi(xP))_m = \sum_{i = 0}^{k - 1} \frac{x_i}{\binom{n}{k - i}} \sum_{\ell \in \text{level }i} q_\ell^m = \sum_{\ell = 0}^{N - 1} (\phi(x))_\ell \cdot q_\ell^m = (\phi(x) Q)_m.
  \end{align*}
\end{proof}

\begin{lemma}\label{th:embedded_spectrum}
  The spectrum $\sigma(P)$ of the matrix $P$ is a subset of the spectrum $\sigma(Q)$ of the matrix $Q$. For any eigenvector $x$ of the matrix $P$ the vector $\phi(x)$ is an eigenvector of $Q$. 
\end{lemma}
\begin{proof}
From Lemma~\ref{lm:phi-linear} and Lemma~\ref{lm:phi-p} it follows that if $x$ is an eigenvector of $P$, then $\phi(x)$ is an eigenvector of $Q$ with the same eigenvalue. Thus, every eigenvalue of $P$ is an eigenvalue of $Q$.
\end{proof}

\begin{lemma}\label{lm:phi-manhattan}
For all $x \in \reals^k$, the Manhattan norm is invariant under $\phi$, that is, $\norm{x}_1 = \norm{\phi(x)}_1$.
\end{lemma}
This follows from the fact that all components of $\phi(x)$ that are from the same level have the same sign.
Notice that an analogous property does not hold for the Euclidean norm $\norm{\cdot}_2$.

Although the two Markov chains represent the same process and each of them contains all information about it, in our analysis we need to use both of them simultaneously. We do not really work with the whole individual space, but only with its subspace $\phi(S)$, where $S$ is the level space. Hence, it is natural to use the terms of the level space to simplify the computations and make them easier to understand. On the other hand, we cannot prove some essential facts about the operator $\mathcal{P}$ represented by the matrix $P$, e.g., that there exists a basis of the level space which consists of the eigenvectors of $\mathcal{P}$ (see Lemma~\ref{th:basis}). To prove them we have to switch to the individual space (or more precisely, to its subspace $\phi(S)$) and use the properties of the self-adjoint operator $\mathcal{Q}$ represented by the symmetric matrix $Q$. Therefore, both chains and their transient matrices are indispensable in our analysis.

\subsection{The Spectrum of the Transient Matrix}\label{spectrum}

The main result of this section is the following analysis of the eigenvalues of~$P$, which builds on the interplay between the two Markov chains.

\begin{lemma}\label{lm:spectrum}
  Let $P$ be the transient matrix of the level chain. Then the following three properties hold.
  \begin{enumerate}
    \item All eigenvalues of $P$ are real.
    \item The largest eigenvalue $\lambda_0$ of $P$ satisfies $\lambda_0 = 1 - O(1/n)$.
    \item Let $\Pr[\alpha = 1] > 0$ and $\Pr[\alpha = 1] = \omega(1/\sqrt[k - 1]{n})$. Then with $c \coloneqq \Pr[\alpha = 1]$ and with $\varepsilon \coloneqq \frac{c^{k - 1}}{(k - 1)2^k}$ any other eigenvalue $\lambda' \ne \lambda_0$ of $P$ satisfies $|\lambda'| < 1 - \varepsilon$.
  \end{enumerate}
\end{lemma}

\begin{proof}

The fact that the eigenvalues are real follows from the facts that by Lemma~\ref{th:embedded_spectrum} the spectrum of $P$ is a subset of the spectrum of $Q$ and that by Lemma~\ref{lm:sym-eigenvalues} all eigenvalues of the symmetric matrix $Q$ are real.

The largest eigenvalue $\lambda_0$ of $P$ is bounded by the minimal and the maximal row sum of $P$ (see Theorem~\ref{lm:p-f}), which are both $1 - O(1/n)$ by Lemma~\ref{lm:rowsum}. 

It remains to show that the absolute values of all other eigenvalues are less than $1 - \varepsilon$ for $\varepsilon = \frac{c^{k - 1}}{(k - 1)2^k}$, which requires more work. To prove this statement we perform a precise analysis of the characteristic polynomial of $P$.

Recall that the spectrum of $P$ is the set of the roots of its characteristic polynomial
\begin{align*}
  \chi_P(\lambda) = \det (P - \lambda I) = \sum\limits_{\sigma \in S_k} \text{sgn}(\sigma) \prod\limits_{i = 0}^{k - 1}(P - \lambda I)_{i, \sigma(i)},
\end{align*}
where $S_k$ is the set of all permutations of the set $[0..k - 1]$ and $\text{sgn}(\sigma)$ denotes the \emph{signature} of permutation $\sigma$ (that is, $+1$ if it can be obtained from the identity permutation in even number of element swaps, and $-1$ otherwise). Note that for all permutations except the identity the product in the sum contains at least one factor $(P - \lambda I)_{i, j}$ with $j > i$ and this element satisfies $(P - \lambda I)_{i, j} = p_i^j = O(1/n)$ by Lemma~\ref{lm:probs}.
The other factors of the product are either $p_{i'}^{j'}$ or $(p_{i'}^{i'} - \lambda)$ for some $i', j',$ therefore every product where $\sigma$ is not the identity is a polynomial in $\lambda$ with coefficients which are $O(1/n)$. Thus, the characteristic polynomial can be written as
\begin{align}
  \chi_P(\lambda) = \prod\limits_{i = 0}^{k - 1} (p_i^i - \lambda) + \beta(\lambda), \label{eq:chi}
\end{align}
where $\beta(\lambda)$ is some polynomial in $\lambda$ with coefficients that are all $O(1/n).$ For this reason the derivative $\beta'(\lambda)$ will also be $O(1/n)$ for all $\lambda \in [-1, 1]$, where we recall that all asymptotics are for $n \to \infty$ (and, e.g., not for any limit behavior of $\lambda$).

To prove that for all eigenvalues $\lambda' \ne \lambda_0$ we have $\lambda' < 1 - \varepsilon$ we need to prove that there is no more than one root of the characteristic polynomial in $[1 - \varepsilon, 1]$.
To do so it suffices to prove that $\chi_P(\lambda)$ is strictly monotonic in this segment.

Consider  $\lambda \ge 1 - \varepsilon$. This implies that $\lambda \ge 1 - \frac{c}{2}$. For every $i \ne 0$ we have $p_i^i \le 1 - \Pr[\alpha = 1] = 1 - c.$ Thus, for every $i \ne 0$ and any $\lambda  \ge 1 - \varepsilon$ we have
\begin{align}\label{eq:chi_roots}
  (p_i^i - \lambda) \le -c/2.
\end{align}

By~\eqref{eq:chi}, the derivative of $\chi_P(\lambda)$ can be written as
\begin{align}\label{eq:chi_derivative}
  \chi_P'(\lambda) = (p_0^0 - \lambda) \left(\prod\limits_{i = 1}^{k - 1} (p_i^i - \lambda)\right)'  - \prod\limits_{i = 1}^{k - 1} (p_i^i - \lambda) + \beta'(\lambda).
\end{align}

Recall that $\varepsilon = \frac{c^{k - 1}}{(k - 1)2^k}$. For all $\lambda \ge 1 - \varepsilon$ we have
\begin{align*}
  p_0^0 - \lambda &\le 1 - \left(1 - \frac{c^{k - 1}}{(k - 1)2^k}\right) = \frac{c^{k - 1}}{(k - 1)2^k}, \\
  \left|\left(\prod\limits_{i = 1}^{k - 1} (p_i^i - \lambda)\right)'\right| &= \left|-\sum_{i = 1}^{k - 1} \prod_{\substack{j \in [1..k-1] \\ j \ne i}} (p_j^j - \lambda)\right| \le (k - 1), \\
  \left|\prod\limits_{i = 1}^{k - 1} (p_i^i - \lambda)\right| &\ge \frac{c^{k - 1}}{2^{k - 1}},
\end{align*}
where the last inequality follows from~\eqref{eq:chi_roots}. Furthermore, from~\eqref{eq:chi_roots} it also follows that for $i \ne 0$ we have $(p_i^i - \lambda) < 0$. Thus,
 \begin{align}\label{eq:signum}
  \sign\left(\left(\prod\limits_{i = 1}^{k - 1} (p_i^i - \lambda)\right)'\right) &= \sign\left(- \sum_{i = 1}^{k - 1}\prod\limits_{\substack{j \in [1..k-1] \\ j \ne i}} (p_j^j - \lambda)\right) = (-1)^{k - 1}.
\end{align}

Consequently, we have two cases.

\underline{Case 1:} When $p_0^0 - \lambda \ge 0$, we have
\begin{align*}
 \bigg| (p_0^0 - \lambda) &\left(\prod\limits_{i = 1}^{k - 1} (p_i^i - \lambda)\right)' -\prod\limits_{i = 1}^{k - 1} (p_i^i - \lambda) \bigg| \\
  &\ge \left|\prod\limits_{i = 1}^{k - 1} (p_i^i - \lambda)\right| - \left| (p_0^0 - \lambda) \left(\prod\limits_{i = 1}^{k - 1} (p_i^i - \lambda)\right)'\right| \\
  &\ge \left|\frac{c}{2}\right|^{k - 1} - \frac{c^{k - 1}}{(k - 1)2^k} \sum_{i = 1}^{k - 1}\left|\prod\limits_{\substack{j \in [1..k-1] \\ j \ne i}} (p_j^j - \lambda)\right| \\
  &\ge \frac{c^{k - 1}}{2^{k - 1}} - (k - 1) \frac{c^{k - 1}}{(k - 1)2^k} = \frac{c^{k - 1}}{2^k}.
\end{align*}

Hence, we have
\begin{align*}
  \left| \chi_P'(\lambda) \right|= \frac{c^{k - 1}}{2^k} + O(1/n) = \frac{\omega(1/n^\frac{1}{k - 1})^{k - 1}}{2^k} + O(1/n) = \omega(1/n).
\end{align*}

\underline{Case 2:} When $p_0^0 - \lambda < 0$, since by~\eqref{eq:chi_roots} $(p_i^i - \lambda) < 0$ for all $i \ne 0$, we have
\begin{align*}
  \left| (p_0^0 - \lambda) \left(\prod\limits_{i = 1}^{k - 1} (p_i^i - \lambda)\right)'  - \prod\limits_{i = 1}^{k - 1} (p_i^i - \lambda) \right| &\ge \left|\prod\limits_{i = 1}^{k - 1} (p_i^i - \lambda)\right| \ge \frac{c^{k - 1}}{2^{k - 1}}.
\end{align*}

Therefore,
\begin{align*}
  \left| \chi_P'(\lambda) \right| = \omega(1/n).
\end{align*}

For $n$ large enough, this together with~\eqref{eq:chi_derivative} and~\eqref{eq:signum} implies that $\chi_P'(\lambda)$ has the same sign as $(-1)^{k - 1}$ for every $\lambda \in [1 - \varepsilon, 1]$. Thus, there can be only one root of characteristic polynomial in this segment.

To rule out that there is a negative eigenvalue $\lambda$ with $|\lambda| > 1 - \varepsilon$, we notice that for $\lambda < -\frac{1}{2}$ and for every $i$ we have $(p_i^i - \lambda) > \frac{1}{2}$.
Therefore, $|\chi_P(\lambda)| > \left(\frac{1}{2}\right)^k - o(1),$ and thus there are no roots that are less than $-\frac{1}{2}$ when $n$ is large enough.

This finally shows that for all eigenvalues $\lambda' \ne \lambda_0$ we have $|\lambda'| < 1 - \min\left(\frac{1}{2}, \frac{c^{k - 1}}{(k - 1)2^k} \right) = 1 - \frac{c^{k - 1}}{(k - 1)2^k}$.

\end{proof}

\section{Runtime Analysis}\label{runtime}

In this section, we prove our main result, which determines the runtime of the \oea on the \plateau function.

\begin{theorem}\label{th:runtime}
  Consider the \oea using any unbiased mutation operator such that the probability to flip exactly one bit is at least $\Pr[\alpha = 1] = \omega\left(n^{-\frac{1}{2k - 2}}\right)$. Let $T$ denote the runtime of this algorithm starting on an arbitrary search point of the plateau of the $\plateau_k$ function. Then
  \begin{align*}
  E[T] &= \frac{n^k}{\Pr[1 \le \alpha \le k] k!}(1 + o(1)),\\
  \Pr[T > t] &= (1 \pm o(1))\left(1 - \frac{k!\Pr[1 \le \alpha \le k]}{n^k}(1 \pm o(1)) \right)^t + r(t),
  \end{align*}
  where $|r(t)| \le \sqrt{N}(1 - \varepsilon)^t$, $\varepsilon = \frac{\left(\Pr[\alpha = 1]\right)^{k - 1}}{(k - 1)2^k}$, and $N = \frac{n^k}{k!}(1 \pm o(1))$. All asymptotic notation refers to $n \to \infty$ and is independent of $t$.
\end{theorem}

We start with a few preparatory results. Recall that by Theorem~\ref{lm:p-f} the largest eigenvalue of a positive matrix has a one-dimensional eigenspace. Also this theorem asserts that both left and right eigenvectors that correspond to the largest eigenvalue have all components with the same sign and they do not have any zero component. Let $\pi^*$ be such a left eigenvector with positive components for $P$ and let it be normalized in such way that $\norm{\pi^*}_1 = 1$.
We view $\pi^*$ as distribution over the levels of the plateau and call it the \emph{conditional stationary distribution of $P$} since it does not change in one iteration under the condition that the algorithm does not find the optimum. Also let $u = (u_0, \dots, u_{k - 1})$ be the probability distribution in the level space such that $\phi(u)$ is the uniform distribution in the individual space. Hence
\begin{align*}
  u_i = \binom{n}{k - i} \bigg/ \sum\limits_{j = 0}^{k - 1} \binom{n}{k - j} = \binom{n}{k - i}N^{-1}
\end{align*}
for all $i \in [0..k-1]$. Our next target is showing that $\pi^*$ and $u$ are asymptotically equal.
For this, we need the following basis of the level space.

\begin{lemma}\label{th:basis}
  There exists a basis $\{e^i\}_{i = 0}^{k - 1}$ of the level space with the following properties.
  \begin{enumerate}
  \item $\pi^* = e^0$.
  \item $e^i$ is an eigenvector of $P$ for all $i \in [0..k-1]$.
  \item The $\phi(e^i)$ are orthogonal in the individual space.
  \end{enumerate}
\end{lemma}

\begin{proof}
Let $S$ be the level space and $S_\text{ind}$ be the individual space. Then $\phi(S)$ is a subspace of $S_\text{ind}$ with $\dim \phi(S) = k,$ since the kernel of $\phi$ is trivial.

Consider the operator $\mathcal{Q}$ that is represented by the matrix $Q$. It is a self-adjoint operator on $S_\text{ind}$, since its matrix is symmetric.
Moreover, this operator maps $\phi(S)$ into $\phi(S)$, since for all $x \in S$ by Lemma~\ref{lm:phi-p} we have $\mathcal{Q}(\phi(x)) = \phi(x) Q = \phi(xP)$.
Therefore, $\mathcal{Q}$ is a self-adjoint operator on $\phi(S)$. Thus, by Lemma~\ref{lm:h-sch} there exists an orthonormal basis $f^0, \dots, f^{k - 1}$ of $\phi(S)$ that consists of eigenvectors of $\mathcal{Q}$.
Let $e^i = \phi^{-1}(f^i)$ for all $i \in [0..k - 1]$. By Lemma~\ref{th:embedded_spectrum} the $e^i$ are eigenvectors of $P$. By the linearity of $\phi$ (Lemma~\ref{lm:phi-linear}) they are linearly independent, hence they form a basis of $S$.

By assuming that $e^0$ corresponds to the largest eigenvalue and multiplying $e^0$ by a suitable scalar, we also satisfy the first property of the lemma.
\end{proof}

We use the basis from Lemma~\ref{th:basis} to prove that $\phi(\pi^*)$ is very close to the uniform distribution.
\begin{lemma}\label{th:uniform}
If $\Pr[\alpha = 1] = \omega\left(n^{-\frac{1}{2k - 2}}\right)$, then for all $j \in [0..k-1]$, we have $\pi_j^* = u_j (1 + \gamma_j)$, where $|\gamma_j| \le \gamma$ for some $\gamma = o(1)$.
\end{lemma}


\begin{proof}

  \begin{figure}[t]
    \centering
    \begin{tikzpicture}

      \node (m) [draw, circle, blue, minimum width=0.2] at (2.35, -0.1) {};
      \node (j) [draw, circle, blue, minimum width=0.2] at (2.25, 1.9) {};

      \node [left] at (0, 0) {Individual space:};
      \node [right] at (0, 0) {$U$};
      \node [right] at (0.4, 0) {$= (U_0, \dots, U_m, \dots, U_{N-1})$};
      \node [right] at (4, 0) {$= \Pi^*$};
      \node [right] at (4.7, 0) {$ +\ U^1 + \ldots $};
      \node [right] at (6.4, 0) {$ +\ U^{k - 1}$};

      \node [left] at (0, 2) {Level space:};
      \node [right] at (0, 2) {$u$};
      \node [right] at (0.4, 2) {$= (u_0, \dots, u_j, \dots, u_{k-1})$};
      \node [right] at (4, 2) {$= \pi^*$};
      \node [right] at (4.7, 2) {$ +\ c_1e^1 + \ldots $};
      \node [right] at (6.4, 2) {$ +\ c_{k - 1}e^{k - 1}$};

      \draw [->] (0.2, 1.8) -- (0.2, 0.2) node [pos=0.5,right] {$\phi$};
      \draw [->] (4.6, 1.8) -- (4.6, 0.2) node [pos=0.5,right] {$\phi$};
      \draw [->] (5.3, 1.8) -- (5.3, 0.2) node [pos=0.5,right] {$\phi$};
      \draw [->] (7, 1.8) -- (7, 0.2) node [pos=0.5,right] {$\phi$};

      \draw [->] (m) -- (j) node [pos=0.5,right, text width=1cm] {individual of level};

      \draw [draw = none] (2.2, 2.2) -- (2.2, 2.4) node [pos=0.5,sloped] {$=$};
      \node [above] at (2.2, 2.4) {$\binom{n}{k - j}N^{-1}$};

      \node (pi) [draw, circle, blue, minimum width=0.5cm] at (5.1, 2.95) {};
      \draw [bend right,->] (pi) to (2.5, 3.2);
      \node [above] at (3.5, 3.4) {$= (1 + \gamma_j) \cdot$};

      \draw [draw = none] (4.6, 2.2) -- (4.6, 2.4) node [pos=0.5,sloped] {$=$};
      \node [above] at (5.2, 2.6) {$(\pi_0^*, \dots, \pi_j^*, \dots, u_{k-1}^*)$};
      \draw [decorate,decoration={brace, amplitude=0.2cm,aspect=0.7}] (6.7, 2.7) -- (3.7, 2.7);

      \draw [draw = none] (2.2, -0.3) -- (2.2, -0.5) node [pos=0.5,sloped] {$=$};
      \node [below] at (2.2, -0.5) {$N^{-1}$};

      \node (Pi) [draw, circle, blue, minimum width=0.6cm] at (5.05, -1.05) {};
      \draw [bend left,->] (Pi) to (2.3, -1.1);
      \node [below] at (3.5, -1.5) {$= (1 + \gamma_j) \cdot$};

      \draw [draw = none] (4.6, -0.3) -- (4.6, -0.5) node [pos=0.5,sloped] {$=$};
      \node [below] at (5.2, -0.8) {$(\Pi_0^*, \dots, \Pi_m^*, \dots, \Pi_{k-1}^*)$};
      \draw [decorate,decoration={brace, amplitude=0.2cm,aspect=0.3}] (3.7, -0.8) -- (6.7, -0.8);
      
      \draw [thick] (-3.5, -2.2) rectangle (9, 4.1);

    \end{tikzpicture}
    \caption{Illustration of the terms and their relations used in Lemma~\ref{th:uniform}}
    \label{fig:illustration}
  \end{figure}
Figure~\ref{fig:illustration} illustrates the relation of the terms used in this proof to make it easier to follow. Lemma~\ref{th:basis}, there exist unique $c_0, c_1, \dots, c_{k - 1} \in \reals$ such that
\[u = \sum_{i = 0}^{k - 1} c_i e^i.\]
 If we transfer this decomposition into the individuals space (using the linearity of $\phi$, see Lemma~\ref{lm:phi-linear}), we obtain
\[U = \sum_{i = 0}^{k - 1} U^i,\]
where we define $U \coloneqq \phi(u)$ and $U^i \coloneqq \phi(c_i e^i)$ for all $i \in [0..k - 1]$. Note that the vector $U$ describes the uniform distribution in the individuals space and hence all its components are equal to $\frac{1}{N}$. For brevity we also define $\Pi^* \coloneqq \phi(\pi^*)$.

We now aim at finding a useful connection between the components of $\Pi^*$ and $U$. Namely, if for all levels $j \in [0..k - 1]$ we prove that for any individual $m$ in level $j$ we have $\Pi_m^* = (1 + \gamma_j) U_m$ with $\gamma_j$ that satisfies the conditions of the theorem, we simultaneously prove the same relation for the components of $\pi^* $ and $u$.

Recall that $\pi^*$ is a normalized vector. Thus, by Lemma~\ref{lm:phi-manhattan} we have $\norm{\Pi^*}_1 = 1$. Recall also that $\pi^* = e^0$ (by the choice of the basis) and therefore, we have $\Pi^* = \frac{U^0}{\norm{U^0}_1}$. For all $m \in [0..N-1]$, we have

\begin{align}\label{eq:pi_star_m}
  \Pi_m^* = \frac{U_m^0}{\norm{U^0}_1}.
\end{align}

The $m$-th component of $U^0$ is
\begin{align*}
  U_{m}^0 = U_{m} - \sum\limits_{i = 1}^{k - 1} U_{m}^i = U_{m} \left(1 - \sum\limits_{i = 1}^{k - 1} \frac{U_{m}^i}{U_{m}}\right) = U_{m} \left(1 - N\sum\limits_{i = 1}^{k - 1} U_{m}^i\right).
\end{align*}
With $\beta_m \coloneqq N\sum\limits_{i = 1}^{k - 1} U_{m}^i$, this simplifies to
\begin{align*}
  U_{m}^0 = U_{m} \left(1 - \beta_m\right) = \frac{1}{N}\left(1 - \beta_m\right).
\end{align*}
We also compute the denominator of~\eqref{eq:pi_star_m} as
\begin{align*}
 \norm{U^0}_1 = \sum_{m=0}^{N-1} |U_{m}^0| = \sum_{m=0}^{N - 1} \frac{1}{N}(1 - \beta_m) = 1 - \sum_{m=0}^{N - 1} \frac{1}{N}\beta_m.
\end{align*}

Putting this into~\eqref{eq:pi_star_m}, we obtain
\begin{align}\label{eq:pi_star_m_beta}
    \Pi_m^* = U_m\frac{1 - \beta_m}{1 - \sum_{m=0}^{N - 1} \frac{1}{N}\beta_m}.
\end{align}

In the remainder of the proof we aim at bounding $|\beta_m|$ from above by some $\beta = o(1)$. Then~\eqref{eq:pi_star_m_beta} gives
\begin{align*}
  \frac{1 - \beta_m}{1 - \sum_{m=0}^{N - 1} \frac{1}{N}\beta_m} \in \left[\frac{1 - \beta}{1 + \beta}, \frac{1 + \beta}{1 - \beta}\right] \subset \left[1 - \frac{2\beta}{1 - \beta}, 1 + \frac{2\beta}{1 - \beta}\right],
\end{align*}
hence defining $\gamma = \frac{2\beta}{1 - \beta} = o(1)$ proves the lemma.

To find the desired $\beta$ with $|\beta_m| < \beta = o(1)$, we regard the vector $U - UQ.$ On the one hand, its elements are very small. For all $m \in [0..N - 1]$, we have
\begin{align*}
  (U - UQ)_m = U_m - \sum\limits_{\ell = 0}^{N - 1} q_\ell^m U_\ell = \frac{1}{N} \left( 1 - \sum\limits_{\ell = 0}^{N - 1} q_m^\ell\right) = \frac{q_m^N}{N},
\end{align*}
where $q_\ell^m$ is the probability to go from individual $\ell$ to individual $m$ (recall that $q_\ell^m = q_m^\ell$) and $q_m^N$ is the probability to leave the plateau from the $m$-th individual. The Euclidean norm of this vector is also very small.
\begin{align*}
\norm{U - UQ}_2 &= \sqrt{\sum\limits_{m = 0}^{N - 1} \left(\frac{q_m^N}{N}\right)^2} \\
                             &= \frac{1}{N} \sqrt{\sum\limits_{i = 0}^{k - 1} \binom{n}{k - i}  \biggl(\Pr[\alpha = k - i] \bigg/ \binom{n}{k - i}\biggr)^2}.
\end{align*}

To bound the sum under the root we notice that it is maximized when we maximize $\Pr[\alpha = 1]$ (since we consider only constant $k$, we assume that $n$ is large enough so that $n > 2k$). Let this probability be equal to $1$, then we have only one non-zero summand and hence,
\begin{align*}
  \norm{U - UQ}_2 &\le \frac{1}{N} \sqrt{\binom{n}{1}^{-1}} = \frac{1}{\sqrt{n}N}.
\end{align*}

On the other hand, if we recall that the $U^i$ are eigenvectors, then we have
\begin{align*}
  U - UQ = \sum\limits_{i = 0}^{k - 1} U^i - \sum\limits_{i = 0}^{k - 1} \lambda_i U^i = \sum\limits_{i = 0}^{k - 1} (1 - \lambda_i) U^i.
\end{align*}

As the $U^i$ are orthogonal, for every $i \in [0..k - 1]$ we have
\begin{align*}
  \norm{(1 - \lambda_i) U^i}_2 \le \norm{U - UQ}_2 \le \frac{1}{\sqrt{n}N}.
\end{align*}

Since the absolute value of every component of a vector cannot be larger than its Euclidean norm, for all $m \in [0..N - 1]$ we conclude that
\begin{align*}
  |(1 - \lambda_i) U_m^i| \le \norm{(1 - \lambda_i) U^i}_2 \le \frac{1}{\sqrt{n}N}.
\end{align*}

Recall that by Lemma~\ref{lm:spectrum} we have
\[
  (1 - \lambda_i) > \varepsilon > \frac{\left(\Pr[\alpha = 1]\right)^{k - 1}}{(k - 1)2^k}
\]
for all $i \ne 0.$ Consequently,
\begin{align*}
  |\beta_m| &= \left|N\sum\limits_{i = 1}^{k - 1} U_m^i\right| \le N\sum\limits_{i = 1}^{k - 1} |U_m^i| \le N\sum\limits_{i = 1}^{k - 1} \frac{1}{\sqrt{n} N (1 - \lambda_i)} \\
  & \le \frac{(k - 1)}{\sqrt{n}\varepsilon} \eqqcolon \beta.
\end{align*}

Since we have $\Pr[\alpha = 1] = \omega(n^{-\frac{1}{2k - 2}})$, we conclude that
\begin{align*}
  \varepsilon = \frac{\left(\omega(n^{-\frac{1}{2k - 2}})\right)^{k - 1}}{(k - 1)2^k} = \omega(1/\sqrt{n}).
\end{align*}

Thus,
\begin{align*}
  \beta = \frac{(k - 1)}{\sqrt{n}\varepsilon} = o(1)
\end{align*}
as desired.

\end{proof}

We are now in the position to prove our main result.

\begin{proof}[Proof of Theorem~\ref{th:runtime}]
  To prove the theorem we first estimate the probability that the runtime $T$ is greater than $t$ by $\Pr[T > t] = \norm{\pi P^t}_1$, where $\pi$ is the initial distribution over the levels of the plateau (see Section~\ref{sec:leaky_chains}).
  Then by~\eqref{eq:expectation-distribution} we estimate the expected runtime as $E[T] = \sum_{t = 1}^{+\infty} \norm{\pi P^{t - 1}}_1$.

To analyse $\norm{\pi P^t}_1$, we decompose $\pi$ into a sum of eigenvectors of $P$ using the basis $e^0, \dots, e^{k - 1}$ from Lemma~\ref{th:basis}. Let $\pi^0, \dots, \pi^{k - 1}$ be scalar multiples of $e^0, \dots, e^{k - 1}$ such that
\begin{align}\label{eq:pi_decomposition}
  \pi = \sum\limits_{i = 0}^{k - 1} \pi^i.
\end{align}
Using the triangle inequalities (Lemma~\ref{lm:triangle}) we obtain
\begin{align*}
\norm{\pi^0 P^t}_1 - \norm{\sum_{i = 1}^{k - 1} \pi^i P^t}_1 \le \norm{\pi P^t}_1 \le \norm{\pi^0 P^t}_1 + \norm{\sum_{i = 1}^{k - 1} \pi^i  P^t}_1.
\end{align*}
Since the $\pi^i$ are the eigenvectors of $P$, we have
\begin{align}\label{eq:pr-runtime-1}
\lambda_0^t \norm{\pi^0}_1 - \norm{\sum_{i = 1}^{k - 1} \lambda_i^t\pi^i }_1 \le \norm{\pi P^t}_1 \le \lambda_0^t\norm{\pi^0 }_1 + \norm{\sum_{i = 1}^{k - 1} \lambda_i^t\pi^i }_1.
\end{align}

Now we estimate the ``error term'' $\norm{\sum_{i = 1}^{k - 1} \lambda_i^t\pi^i }_1$ of these bounds.
First, by Lemma~\ref{lm:phi-manhattan} and by Lemma~\ref{lm:phi-linear}, we have
\begin{align}\label{eq:pr-runtime-2}
\norm{\sum_{i = 1}^{k - 1} \lambda_i^t\pi^i }_1 = \norm{\phi\left(\sum_{i = 1}^{k - 1} \lambda_i^t\pi^i\right) }_1 = \norm{\sum_{i = 1}^{k - 1} \lambda_i^t\phi(\pi^i) }_1.
\end{align}
Using Lemma~\ref{lm:norm-manhattan-euclidean} and Lemma~\ref{lm:euclidean-norm-1} we estimate
\begin{align}\label{eq:pr-runtime-3}
\norm{\sum_{i = 1}^{k - 1} \lambda_i^t\phi(\pi^i) }_1 \le \sqrt{N} \norm{\sum_{i = 1}^{k - 1} \lambda_i^t\phi(\pi^i) }_2 \le \sqrt{N} \max_{i \in [1..k - 1]}(|\lambda_i^t|)\norm{\sum_{i = 1}^{k - 1} \phi(\pi^i) }_2.
\end{align}
By Lemma~\ref{lm:spectrum} we have $\max_{i \in [1..k - 1]}(|\lambda_i^t|) \le (1 - \varepsilon)^t$, where $\varepsilon = \frac{\left(\Pr[\alpha = 1])\right)^{k - 1}}{(k - 1)2^k}$. Hence, by Lemma~\ref{lm:euclidean-norm-2} and Lemma~\ref{lm:norm-manhattan-euclidean}, we conclude
\begin{align}\label{eq:pr-runtime-4}
\sqrt{N} \max_{i \in [1..k - 1]}(|\lambda_i^t|)\norm{\sum_{i = 1}^{k - 1} \phi(\pi^i) }_2 \le \sqrt{N}(1 - \varepsilon)^t \norm{\phi(\pi)}_1 = \sqrt{N}(1 - \varepsilon)^t.
\end{align}
Finally, by~\eqref{eq:pr-runtime-1},~\eqref{eq:pr-runtime-2},~\eqref{eq:pr-runtime-3}, and~\eqref{eq:pr-runtime-4} we obtain that

\begin{align}\label{eq:pr-runtime}
\norm{\pi P^t}_1 = \lambda_0^t \norm{\pi^0}_1 + r(t),
\end{align}
where $|r(t)| \le \sqrt{N}(1 - \varepsilon)^t$.

To estimate the expected runtime we put~\eqref{eq:pr-runtime} into~\eqref{eq:expectation-distribution} and obtain
\begin{align}\label{eq:expected-runtime}
\begin{split}
  E[T] &= \sum\limits_{t = 1}^{+\infty}\norm{\pi P^{t - 1}}_1 = \sum\limits_{t = 1}^{+\infty}\left(\lambda_0^{t - 1} \norm{\pi^0}_1 + r(t - 1) \right) \\
       &= \sum\limits_{t = 1}^{+\infty} \lambda_0^{t - 1} \norm{\pi^0}_1 + \sum\limits_{t = 1}^{+\infty} r(t - 1) = \frac{\norm{\pi^0}_1}{1 - \lambda_0} + \sum\limits_{t = 0}^{+\infty} r(t).
\end{split}
\end{align}
By~\eqref{eq:pr-runtime-4} we have
\begin{align*}
  \left|\sum\limits_{t = 0}^{+\infty} r(t) \right| \le \sum\limits_{t = 0}^{+\infty} |r(t)| \le \sum\limits_{t = 1}^{+\infty} \sqrt{N}(1 - \varepsilon)^t = \frac{(1 - \varepsilon)\sqrt{N}}{\varepsilon} \le \frac{\sqrt{N}}{\varepsilon}.
\end{align*}
From the assumptions of the theorem we have
\begin{align*}
\varepsilon = \frac{(\Pr[\alpha = 1])^{k - 1}}{(k - 1)2^k} = \frac{\omega(1/\sqrt{n})}{(k - 1)2^k} = \omega(1/\sqrt{n}),
\end{align*}
and hence
\begin{align*}
 \left|\sum\limits_{t = 0}^{+\infty} r(t) \right| = o(\sqrt{nN}).
\end{align*}

It remains to estimate $\norm{\pi^0}_1$ and $\lambda_0$. Recall that by Lemma~\ref{lm:phi-manhattan} we have $\norm{\pi^0}_1 = \norm{\phi(\pi^0)}_1$.
From the linearity of $\phi$ (see Lemma~\ref{lm:phi-linear}) and~\eqref{eq:pi_decomposition} we obtain $\phi(\pi) = \sum_{i = 0}^{k - 1}\phi(\pi^i)$.
Since $\pi^0, \dots \pi^{k - 1}$ are scalar multiples of $e^0, \dots e^{k - 1}$ and all $\phi(e^i)$ are orthogonal, we have a decomposition of $\phi(\pi)$ into a sum of orthogonal vectors.
Therefore, by Lemma~\ref{lm:projection}, we have

\begin{align}\label{eq:coordinate}
  \norm{\phi(\pi^0)}_1 = \frac{\langle\phi(\pi), \phi(e^0)\rangle}{\langle\phi(e^0), \phi(e^0)\rangle}.
\end{align}

By Lemma~\ref{th:uniform}, the components of $e^0$ are almost equal to the components of the vector $u$ of the uniform distribution in the level space\footnote{Note that $e^0$ is the same vector as $\pi^*$ in Lemma~\ref{th:uniform}. However we now refer to this vector as $e^0$ to underline that we consider it as a basis vector of the level space, while in Lemma~\ref{th:uniform} we referred to it as $\pi^*$ since we considered it as a vector of the probabilistic distribution over the states of the level chain.}. Transferring this result to the individual space, we have $(\phi(e^0))_m = \frac{1}{N}(1 + \gamma_j)$ for all $j \in [0.. k - 1]$ and all individuals $m$ that belong to level $j$.
This and the fact that by Lemma~\ref{lm:phi-manhattan} we have $\norm{\phi(\pi)}_1 = \norm{\pi}_1 = 1$ allow to calculate the inner products in~\eqref{eq:coordinate} and obtain
\begin{align}\label{eq:pi}
  \norm{\pi^0}_1 = \frac{\sum\limits_{j = 0}^{k - 1}\sum\limits_{m \in \text{level }j} (\phi(\pi))_m \frac{1 + \gamma_j}{N}}{\sum\limits_{j = 0}^{k - 1}\sum\limits_{m \in \text{level }j}  \left( \frac{1 + \gamma_j}{N}\right)^2} = 1 + o(1).
\end{align}

We compute $(1 - \lambda_0)$ in the following way. First, since $\norm{\pi^*}_1 = 1$ and $\pi^*$ is an eigenvector of $P$, we have
\begin{align*}
  1 - \lambda_0 &= 1 - \norm{\lambda_0\pi^*}_1 = 1 - \norm{\pi^* P}_1 = 1 - \sum\limits_{i = 0}^{k - 1} \left|\sum\limits_{j = 0}^{k - 1} \pi_j^* p_j^i \right|.
\end{align*}

Since by Theorem~\ref{lm:p-f} all components of $\pi^*$ are positive and the components of $P$ are non-negative, this simplifies to

\begin{align*}
1 - \sum\limits_{i = 0}^{k - 1} \left|\sum\limits_{j = 0}^{k - 1} \pi_j^* p_j^i \right| = 1 - \sum\limits_{i = 0}^{k - 1} \sum\limits_{j = 0}^{k - 1} \pi_j^* p_j^i = 1 - \sum\limits_{j = 0}^{k - 1} \pi_j^* \sum\limits_{i = 0}^{k - 1} p_j^i.
\end{align*}

By the definition of $P$, the sum of the $j$-th row of $P$ is equal to $(1 - p_j^k)$, and we have $\sum_{i = 0}^{k - 1} \pi_i^* = \norm{\pi^*}_1 = 1$. Hence,
\begin{align*}
1 - \sum\limits_{j = 0}^{k - 1} \pi_j^* \sum\limits_{i = 0}^{k - 1} p_j^i = 1 - \sum\limits_{j = 0}^{k - 1} \pi_j^* (1 - p_j^k) = \sum\limits_{j = 0}^{k - 1} \pi_j^* p_j^k.
\end{align*}

By Lemma~\ref{th:uniform} and~\eqref{eq:trans_prob} we have
\begin{align*}
  \sum\limits_{j = 0}^{k - 1} \pi_j^* p_j^k &= \sum\limits_{j = 0}^{k - 1} \binom{n}{k - j}N^{-1}(1 + \gamma_j) \binom{n}{k - j}^{-1} \Pr[\alpha = k - j] \\
                                            &= \frac{1}{N} \sum\limits_{j = 1}^{k} \Pr[\alpha = j] (1 + \gamma_{k - j}) = \frac{1}{N}\Pr[1 \le \alpha \le k] (1 + o(1)).
\end{align*}

Thus, we obtain
\begin{align}\label{eq:lambda_0}
\lambda_0 = 1 - \frac{1}{N}\Pr[1 \le \alpha \le k] (1 + o(1)).
\end{align}

By substituting $\lambda_0$ and $\norm{\pi^0}_1$ into~\eqref{eq:expected-runtime} and~\eqref{eq:pr-runtime} with their values from~\eqref{eq:lambda_0} and~\eqref{eq:pi} and recalling that $N = \frac{n^k}{k!}(1 + o(1))$, we prove the theorem.
\end{proof}

We also underline that $r(t)$ in the tail bounds on the runtime distribution is negligible, as soon as $t = \omega(\sqrt{n}\log(n))$, that is, far before the algorithm finds the optimum.

\section{Corollaries}\label{sec:cor}

We now exploit Theorem~\ref{th:runtime} to analyze how the choice of the mutation operator influences the runtime. 

By the runtime in this section we mean the runtime of the algorithm when it starts from an arbitrary individual, which is not necessarily on the plateau. However, without proof we notice that the \oea with any mutation operator considered in this section reaches the plateau in an expected number of $O(n \log(n))$ iterations from any starting individual, which is significantly less than the time which it spends on the plateau. Therefore, the time to leave the plateau coincides with the total runtime precisely apart from lower order terms. Since, by our main result, the time to leave the plateau depends only on the probability to flip between $1$ and $k$ bits, determining the runtimes in this section  is an easy task.

We first observe that for all unbiased operators with constant probability to flip exactly one bit, the expected optimization time is $\Theta(N)$, where we recall that the size $N$ of the plateau is \[N=\sum_{i=0}^{k-1} \binom{n}{n-k+i} = (1 \pm o(1)) \frac{n^k}{k!}.\]
Hence all these mutation operators lead to asymptotically the same runtime of $\Theta(n^k)$. The interesting aspect thus is how the leading constant changes.

\subsection{Randomized Local Search and Variants}

When taking such a more precise look at the runtime, that is, including the leading constant, then the best runtime, obviously, is obtained from mutation operators which flip always between $1$ and $k$ bits. This includes variants of randomized local search which also flip more than one bit, see, e.g.,~\cite{GielW03,NeumannW07,DoerrDY20}, as long as they do not flip more than $k$ bits, but most prominently the classic randomized local search heuristic, which always flips a single random bit. Note that the latter uniformly for all $k$ (and including the case $k=1$ not regarded in this work) is among the most effective algorithms.

\subsection{Standard \oea}

The classic mutation operator in evolutionary computation is \emph{standard bit mutation}, where each bit is flipped independently with some probability (``mutation rate'') $\gamma/n$, where $\gamma$ usually is a constant. We call the \oea which uses the standard bit mutation the \emph{standard \oea}.

\begin{theorem}
  Let $\gamma$ be some arbitrary positive constant and $k \ge 2$. Then the standard \oea with mutation rate $\gamma/n$ optimizes $\plateau_k$ in an expected number of \[E[T] = (1 + o(1)) \frac{n^k}{k! e^{-\gamma}\sum_{i = 1}^k \frac{\gamma^i}{i!}}\] iterations. This time is asymptotically minimal for $\gamma = \sqrt[k]{k!} \approx k/e$.
\end{theorem}

\begin{proof}
For the standard bit mutation with mutation rate $\gamma/n$, the probability to flip exactly one bit is
\begin{align*}
  n \frac{\gamma}{n} \left(1 - \frac{\gamma}{n}\right)^{n - 1} \ge \gamma e^{-\gamma} (1-o(1)),
\end{align*}
which is at least some positive constant as long as $\gamma$ is a constant. Thus, we can apply Theorem~\ref{th:runtime} and obtain
\begin{align*}
 E[T] &= (1 \pm o(1))\frac{N}{\Pr[1 \le \alpha \le k]} \\
 &= (1 \pm o(1))N\left(\sum\limits_{i = 1}^k \binom{n}{i} \left(\frac{\gamma}{n}\right)^i \left(1 - \frac{\gamma}{n}\right)^{n - i}\right)^{-1} \\
 &= (1 \pm o(1))\frac{n^k}{k!}\left(\sum\limits_{i = 1}^k \frac{\gamma^i}{i!} e^{-\gamma}  \right)^{-1}.
\end{align*}

Consider $d(\gamma) = e^{-\gamma} \sum_{i = 1}^k \frac{\gamma^i}{i!}$. In order to minimize $E[T]$, we have to maximize $d(\gamma)$. Now $\gamma \mapsto d(\gamma)$ is a smooth continuous function, so its maximal value for $\gamma \in [0, +\infty)$ can only be at $\gamma = 0$, for $\gamma \to +\infty$, or in the zeros of its derivative. We have $d(0) = \lim_{\gamma \to \infty} d(\gamma) = 0$. The  derivative is
\begin{align*}
d'(\gamma) &= \left(e^{-\gamma} \sum\limits_{i = 1}^k \frac{\gamma^i}{i!}\right)'
       = e^{-\gamma} \sum\limits_{i = 1}^k \frac{i\gamma^{i - 1}}{i!} - e^{-\gamma} \sum\limits_{i = 1}^k \frac{\gamma^i}{i!} \\
       &= e^{-\gamma}\left( \sum\limits_{i = 0}^{k - 1} \frac{\gamma^i}{i!} - \sum\limits_{i = 1}^k \frac{\gamma^i}{i!} \right) = e^{-\gamma} \left(1 - \frac{\gamma^k}{k!}\right).
\end{align*}
Hence the only value of $\gamma$ with $d'(\gamma) = 0$ is $\gamma = \sqrt[k]{k!}$. For this value we have $d(\sqrt[k]{k!}) > 0$, so this defines the unique optimal mutation rate. Finally, by Stirling's formula $k! \approx \sqrt{2\pi k} \left(\frac{k}{e}\right)^k$ we have
$
 \sqrt[k]{k!} \approx (2\pi k)^\frac{1}{2k} \frac{k}{e} \approx \frac{k}{e}.
$
\end{proof}

\subsection{Fast \oea}

The \emph{fast \oea} recently proposed in~\cite{doerr-fast-ga} is simply a \oea that uses standard bit mutation with a random mutation rate $\gamma/n$ with $\gamma \in [1..n/2]$ chosen according to a power-law distribution. More precisely, for a parameter $\beta > 1$ which is assumed to be a constant (independent of $n$), we have
\begin{align*}
  \Pr[\gamma = i] = 0
\end{align*} for every $i > n/2$ and $i = 0$, and
\begin{align*}
  \Pr[\gamma = i] = i^{-\beta} / H_{n/2, \beta}
\end{align*}
otherwise, where $H_{n/2, \beta} := \sum_{i=1}^{n/2} i^{-\beta}$ is a generalized harmonic number.

\begin{theorem}\label{th:fga}
  For $k \ge 2$ the expected runtime of the fast \oea on $\plateau_k$ is $C_{kn} \frac{n^k}{k!}$, where $C_{kn} \coloneqq \frac{H_{n/2, \beta}}{H_{k, \beta}}(1 + o(1))$ can be bounded by constants, namely $C_{kn} \in \left[\frac{\frac{1}{\beta - 1} - o(1)}{H_{k,\beta}}, \frac{\frac{1}{\beta - 1} + 1}{H_{k,\beta}}\right]$.
\end{theorem}

\begin{proof}
  From the definition of the fast \oea we have
  \begin{align*}
    \sum\limits_{i = 1}^k \Pr[\gamma = i] = \frac{\sum\limits_{i = 1}^k i^{-\beta}}{\sum\limits_{i = 1}^{n/2} i^{-\beta}} = \frac{H_{k, \beta}}{H_{n/2, \beta}}.
  \end{align*}

  Since $\beta > 1$ and $k$ are constants,  $H_{k, \beta}$ is a constant as well. We estimate $H_{n/2, \beta}$ through the corresponding integral.
  \begin{align*}
    \frac{1}{\beta - 1}+1 = \int\limits_{1}^{+\infty} x^{-\beta} dx + 1 \ge H_{n/2, \beta} \ge \int\limits_{1}^{n/2} x^{-\beta} dx = \frac{1 - (n/2)^{1 - \beta}}{\beta - 1} = \frac{1 - o(1)}{\beta - 1}.
  \end{align*}
  Notice that $\Pr[\alpha = 1] = \frac{H_{1, \beta}}{H_{n/2, \beta}} = \left(H_{n/2, \beta}\right)^{-1}$ is at least some constant. Thus, Theorem~\ref{th:runtime} gives an expected runtime of $E[T] = \frac{H_{n/2, \beta}}{H_{k, \beta}}N(1 + o(1))$, which we can estimate by
    \begin{align*}
       \frac{\frac{1-o(1)}{\beta - 1}}{H_{k, \beta}}\frac{n^k}{k!} \le E[T] \le \frac{\frac{1}{\beta - 1} + 1}{H_{k, \beta}}\frac{n^k}{k!}(1 + o(1)).
    \end{align*}
\end{proof}

\subsection{Hyper-Heuristics}\label{sec:hh}

Hyper-heuristics are randomized search heuristics that combine, in a suitable and again usually randomized fashion, simple low-level heuristics. Despite many success stories in applications, their theoretical understanding is still very low and only the last few years have seen some first results. These exclusively regard simple $(1+1)$ type hill-climbers which choose between different mutation operators as low-level heuristics. We now regard the hyper-heuristics discussed in~\cite{AlanaziL14} argue that for some of these, our method is applicable, whereas for others it is not clear how to do this.

Like almost all previous theoretical works, we regard as available low-level mutation operators one-bit flips (flipping a bit chosen uniformly at random) and two-bit flips (flipping two bits chosen uniformly at random from all 2-sets of bit positions). Hence the $(1+1)$ hill-climber with this a hyper-heuristic selection between these two operators starts with a random search point and then repeats generating a new search point by applying one of the mutation operators (chosen according to the hyper-heuristic) and accepting the new search point if it has an at least as good fitness as the parent.

The most elementary hyper-heuristic called \emph{simple random} in each iteration simply chooses one of the two available mutation operators with equal probability $1/2$. This compound mutation operator (choosing one randomly and applying it) still is a unary unbiased mutation operator, so our main result (Theorem~\ref{th:runtime}) is readily applicable and gives the following result.

\begin{theorem}
  Consider the $(1+1)$ hill-climber using the \emph{simple random} hyper-heuristic to decide between the one-bit flip and the two-bit flip mutation operator. When started on an arbitrary point of the plateau, its runtime $T$ on the $\plateau_k$, $k \ge 2$, function satisfies
  \[E[T] = \frac{n^k}{k!}(1 \pm o(1)) .\]
\end{theorem}

The more interesting hyper-heuristic \emph{random gradient} in the first iteration chooses a random low-level heuristic. In each further iteration, it chooses the same low-level heuristic as in the previous iteration, if this has ended with a fitness gain, and it chooses again a random low-level heuristic otherwise. This way of performing mutation obviously cannot be described via a single unary unbiased operator. However, once the algorithm has reached the plateau, it can. The reason is that from that point on and until the optimum is found, no further improvements are found. Consequently, the algorithm reverts to the one using the \emph{simple random} approach.

\begin{corollary}
  Consider the $(1+1)$ hill-climber using the \emph{random gradient} hyper-heuristic to decide between the one-bit flip and the two-bit flip mutation operator. When started on an arbitrary point of the plateau, its runtime $T$ on the $\plateau_k$, $k \ge 2$, function satisfies
  \[E[T] = \frac{n^k}{k!}(1 + o(1)) .\]
\end{corollary}

For two other common hyper-heuristics, we currently do not see how to apply our methods. The \emph{permutation} heuristic initially fixes a permutation of the low-level heuristics and then repeatedly uses them in this order. The \emph{greedy} heuristic uses, in each iteration, all available hyper-heuristics in parallel and proceeds with the best offspring produced (if it is at least as good as the parent). While we are optimistic that these heuristics lead to asymptotically the same runtimes as the two heuristics just analyzed, we cannot prove this since our main result is not applicable.

It has been observed in~\cite{LissovoiOW17} that, due to the generally low probability of finding an improvement, better results are obtained when the \emph{random gradient} heuristic is used with a longer learning period, that is, the randomly chosen low-level heuristic is repeated for a phase of $\tau$ iterations. If an improvement is found, a new phase with the same low-level heuristic is started. Otherwise, the next phase starts with a random operator. This idea was extended in~\cite{DoerrLOW18} so that now a phase was called successful if within $\tau$ iterations a certain number $\sigma$ of improvements were obtained. This mechanism was more stable and allowed a self-adjusting choice of the previously delicate parameter $\tau$. Again, for these hyper-heuristics our results are not applicable.

\subsection{Comparison for Concrete Values}

Since the leading constants computed above, in their general form, are hard to compare, we now provide in Table~\ref{tbl:comparison} a few explicit values for specific algorithm parameters and plateau sizes. 

\begin{table}[h]
\begin{center}
\begin{tabular}{|p{5cm}|p{3cm}||c|c|c|}
\hline
\multicolumn{2}{|l||}{Algorithm} & $k = 2$ & $k = 4$ & $k = 6$ \\ \hline\hline
\multicolumn{2}{|l||}{Random Local Search} & 1 & 1 & 1 \\ \hline
\multirow{2}{5cm}{\oea with standard bit mutation} & Mutation rate $1/n$ & $1.812$ & $1.591$ & $1.582$ \\ \cline{2-5}
& Mutation rate $k/(en)$ & $2.074$ & $1.328$ & $1.027$ \\ \hline
\multirow{2}{5cm}{Fast Genetic Algorithm} & $\beta = 1.5$ & $1.930$ & $1.563$ & $1.428$  \\ \cline{2-5}
& $\beta = 2$ & $1.316$ & $1.155$ & $1.103$ \\ \hline
\multirow{2}{5cm}{\oea with hyperheuristics} & simple random & 1 & 1 & 1 \\ \cline{2-5}
& random gradient & 1 & 1 & 1 \\ \hline
\end{tabular}
\end{center}
\caption{Comparison of the leading constant in the expected runtime of the evolutionary algorithms with different mutation operators on the $\plateau_k$ function, that is, the constant $c$, such that the expected runtime is $c\frac{n^k}{k!}(1 - o(1))$.}
\label{tbl:comparison}
\end{table}

\section{Conclusion}\label{conclusion}

In this paper we developed a new method to analyze the runtime of evolutionary algorithms on plateaus. This method does not depend on the particular mutation operator used by the EA as long as there is a sufficiently large probability to flip a single random bit. We performed a very precise analysis on the particular class of plateau functions, but we are optimistic that similar methods can be applied for the analysis of other plateaus. For example, Lemmas~\ref{lm:spectrum},~\ref{th:basis} and~\ref{th:uniform} remain true for those plateaus of the function \xdivk (that is defined as $\lfloor\onemax(x) / k\rfloor$ for some parameter $k$) that are in a constant Hamming distance from the optimum (and these are the plateaus which contribute most to the runtime). That said, the proof of Lemma~\ref{lm:spectrum} would need to be adapted to these plateaus different from the one of our plateau function. We are optimistic that this can be done, but leave it as an open problem for now.

The inspiration for our analysis method stems from the observation that the algorithm spends a relatively long time on the plateau. So regardless of the initial distribution on the plateau, the distribution of the individual converges to the conditional stationary distribution long before the algorithm leaves the plateau. This indicates that our method is less suitable to analyze how evolutionary algorithms leave plateaus which are easy to leave, but such plateaus usually present not bigger problems in optimization.

Overall, we are optimistic that our main analysis method, switching between the level chain and the individual chain, which might be the first attempt to devise a general analysis method for EAs on plateaus, will find further applications.

While our analysis method can deal with a large class of $(1+1)$-type hill-climbers, it is currently less clear how to analyze population-based algorithms. A series of works~\cite{he-yao-drift-intro,HeY04,JansenJW05,Witt06,JagerskupperS07,Chen09,RoweS14,DoerrK15,AntipovDFH18,AntipovDY19,DoerrDE15,DoerrD18} analyzing the runtime of various versions of the \ea, \commaea, and \ollga on \onemax show that these algorithms quickly reach the plateau of the \plateau function, but it is currently not clear how to extend our method to get sharp runtime estimates also for the part of the process on the plateau. Likewise, it is not clear how our methods can be extended to algorithms that dynamically change their parameters~\cite{DoerrD20bookchapter}, because here in most cases the relevant state of the algorithm not only consists of the current search point(s). For hyper-heuristics, we could show two elementary results, but again, as discussed in Section~\ref{sec:hh}, for most hyper-heuristics our general result cannot be applied. By analogy with jump functions, crossover-based algorithms should be efficient on plateaus, especially the ones using different diversity mechanisms~\cite{DangFKKLOSS18}. However, these algorithms are more complicated, and even on jump functions there are no asymptotically tight bounds on their runtime. This suggest that studying their behavior on plateaus might be even more complicated. Given that plateaus of constant fitness appear frequently in optimization problems, we feel that the open questions discussed in this paragraph are worth pursuing in the near future.

\section*{Acknowledgements}

This work was financially supported by the National Center for Cognitive Research of ITMO University and by a public grant as part of the Investissements d'avenir project, reference ANR-11-LABX-0056-LMH, LabEx LMH.

\appendix
\section{Appendix}

Since both a reviewer of our submission to the GECCO 2018 theory track and (after rejection of the former) a reviewer of PPSN 2018 claimed that our result is already wrong for the small case $k=2$ and a specific mutation operator, to clarify the situation and to avoid similar problems in future reviewing processes, we analyze now the case $k=2$ in full generality 
by elementary means. This proves the reviewers' claims wrong and shows that the case $k=2$ can be solved by regarding a simple system of equations, whose unique solution agrees with our main results.

We start by describing the reviewers' incorrect concerns. The reviewer of our submission to the GECCO 2018 theory track wrongfully considers the following a counterexample. Suppose $k = 2$ and suppose the mutation operator flips either $1$ or $2$ randomly chosen bits, each with probability $\frac{1}{2}$. Then our theorem gives an expected runtime of $E[T] = n^2/2 (1 + o(1))$, but the reviewer claims that the expected runtime is $n^2(1 + o(1))$, referring to own calculations not provided.

The reviewer of PPSN 2018 suggested a more general counterexample. She or he considers again $k = 2$ and the mutation operator that flips one randomly chosen bit with probability $p_1$ (where $p_1 = \Omega(1)$) and it flips two randomly chosen bits with probability $p_2$. The reviewer claims, again without giving details, that the expected runtime of the described algorithm on the $\plateau_2$ function is $ \frac{n^2}{2p_1 + p_2}(1 + o(1))$, while our Theorem~\ref{th:runtime} gives an expected runtime of $E[T] = \frac{n^2}{2(p_1 + p_2)}(1 + o(1))$.

To cover both examples and possible future ones, we consider the case $k = 2$ for a general unbiased mutation operator (with probability to flip exactly one bit of $\Omega(1)$). For brevity and to match the notation of the latter reviewer we define $p_i$ as the probability that the mutation operator flips exactly $i$ bits (that was denoted as $\Pr[\alpha = i]$ in the main part of the paper). We recall from the body of the paper that there is a one-to-one correspondence between unary unbiased mutation operators and vectors $p = (p_0,p_1, \dots, p_n) \in [0,1]^{n+1}$ with $\|p\|_1=1$.

\begin{lemma}
  If $p_1 = \Omega(1)$, then the runtime of the \oea with an arbitrary unbiased mutation operator as described above optimizing the $n$-dimensional $\plateau_2$ function is $\frac{n^2}{2(p_1 + p_2)}(1 + o(1))$.
\end{lemma}
\begin{proof}
To find the expected runtime of the algorithm on the plateau we consider the level chain. It contains two states (for level $0$ and level $1$), but additionally to find the expected runtime we now include the optimum into the chain as a new state called level $2$.

By $p_i^j$ we denote the transition probabilities between levels for $i \in [0..1]$ and $j \in [0..2]$. We define $T_i$ as the runtime of the algorithm if it starts on level $i$ for $i \in [0..1]$. The following system of equations follows from elementary Markov chain theory.

\begin{align*}
E[T_0] &= 1 + p_0^0 E[T_0] + p_0^1 E[T_1]. \\
E[T_1] &= 1 + p_1^0 E[T_0] + p_0^0 E[T_1]. \\
\end{align*}

By elementary transformations we obtain the following equivalent system. 

\begin{equation}\label{eq:runtme_expressions}
\begin{split}
E[T_0] &= \frac{p_1^0 + p_0^1 + p_1^2}{p_1^2 p_0^1 + p_1^0 p_0^2 + p_1^2 p_0^2}\,. \\
E[T_1] &= \frac{p_1^0 + p_0^1 + p_0^2}{p_1^2 p_0^1 + p_1^0 p_0^2 + p_1^2 p_0^2}\,. \\
\end{split}
\end{equation}

To evaluate the right-hand sides, we first compute all the transition probabilities.

\begin{itemize}
\item $p_0^1$ is the probability to either flip one zero-bit or to flip both zero-bits and one one-bit, that is,
$$p_0^1 = p_1 \frac{2}{n} + p_3 \frac{6}{n(n - 1)} = p_1 \frac{2}{n} + o(1/n).$$
Recall that $p_1$ is considered as some positive constant.
\item $p_0^2$ is the probability to flip both zero-bits, that is,
$$p_0^2 = p_2 \frac{2}{n(n - 1)}.$$
\item $p_1^0$ is the probability to either flip one one-bit or to flip two one-bits and the only zero-bit, that is,
$$p_1^0 = p_1 \frac{n - 1}{n} + p_3 \frac{3}{n} = p_1 + o(1).$$
\item $p_1^2$ is the probability to flip the only zero-bit, that is,
$$p_1^2 = p_1 \frac{1}{n}.$$
\item In other cases the mutation operator generates either an individual with the same number of one-bits or an individual from outside the plateau, so the algorithm does not accept it. Therefore, $p_0^0 = 1 - p_0^1 - p_0^2$ and $p_1^1 = 1 - p_1^0 - p_1^2$.
\end{itemize}
We compute the numerators and denominator in the right-hand sides of~\eqref{eq:runtme_expressions}.
\begin{itemize}
\item $p_1^0 + p_0^1 + p_1^2 = p_1 + o(1) + p_1 \frac{2}{n} + o(1/n) + p_1 \frac{1}{n} = p_1 + o(1).$
\item $p_1^0 + p_0^1 + p_0^2 = p_1 + o(1) + p_1 \frac{2}{n} + o(1/n) + p_2 \frac{2}{n(n - 1)} = p_1 + o(1).$
\item $p_1^2 p_0^1 + p_1^0 p_0^2 + p_1^2 p_0^2 = (p_1)^2 \frac{2}{n^2} + p_1 p_2 \frac{2}{n^2} + o(1/n^2)$.
\end{itemize}
This gives the desired values for the expected runtimes.
\begin{align*}
E[T_0] &= \frac{p_1 + o(1)}{p_1(p_1 + p_2)\frac{2}{n^2} + o(1/n^2)} = \frac{n^2}{2(p_1 + p_2)}(1 + o(1)),\\
E[T_1] &= \frac{p_1 + o(1)}{p_1(p_1 + p_2)\frac{2}{n^2} + o(1/n^2)} = \frac{n^2}{2(p_1 + p_2)}(1 + o(1)). \\
\end{align*}

So independently on the starting state we have precisely the same expected runtime (apart from the lower order terms ignored in both cases) as obtained through Theorem~\ref{th:runtime}.

\end{proof}


\newcommand{\etalchar}[1]{$^{#1}$}

\end{document}